\documentclass{article}\usepackage{amsthm}

\DeclareMathSizes{10.5}{10.5}{5}{5}
\newtheorem{thm}{Theorem}[section]
\usepackage{etoolbox}
\newtoggle{arxiv}
\toggletrue{arxiv}
\usepackage{graphicx} 
\usepackage{tabularx} 
\usepackage{algorithmicx}
\usepackage[noend]{algpseudocode}
\usepackage{float}
\usepackage{color}
\usepackage{natbib,natbibspacing}
\usepackage{multicol}
\usepackage{algorithm}
\usepackage{color}
\usepackage{url}
\usepackage{amssymb}
\usepackage{xspace}
\usepackage{amsmath}
\usepackage{alltt}
\usepackage{lmodern}
\usepackage{changepage}

\newtheorem{thmLem}{Theorem}[section]
\newtheorem{lem}[thmLem]{Lemma}

\newtheorem{thmDef}{Theorem}[section]
\newtheorem{defn}[thmDef]{Definition}
\newtheorem{thmRes}{Theorem}[section]
\newtheorem{res}[thmRes]{Result}
\newtheorem{thmConj}{Theorem}[section]
\newtheorem{conj}[thmConj]{Conjecture}
\newcommand{\n}[1]{\ensuremath{\widetilde{#1}}\xspace}

\newcommand{\bb}[1]{\ensuremath{\left(#1\right)}}

\newcommand{\var}{\text{var}}

\newcommand{\diag}{\text{diag}}

\newcommand{\natwo}{\ensuremath{\n{A}'}}
\newcommand{\purna}[1]{}

\newcommand{\er}{Erd\H{o}s-R\'{e}nyi\xspace}

\newcommand{\lsum}{\ensuremath{\sum\limits}\xspace}

\newcommand{\sbm}{Stochastic Blockmodel\xspace}

\let\oldh\hat
\renewcommand{\hat}[1]{\ensuremath{\widehat{#1}}\xspace}

\newcommand{\e}{\ensuremath{\mathbf{e}}\xspace}
\newcommand{\uv}{\ensuremath{\mathbf{u}}\xspace}
\newcommand{\vv}{\ensuremath{\mathbf{v}}\xspace}
\newcommand{\xv}{\ensuremath{\mathbf{x}}\xspace}
\newcommand{\wv}{\ensuremath{\mathbf{w}}\xspace}

\newcommand{\na}{\ensuremath{\n{A}}\xspace}
\newcommand{\nao}{\ensuremath{\n{A}_1}\xspace}
\newcommand{\ph}{\ensuremath{\oldh{p}}\xspace}
\newcommand{\rb}{\texttt{RB}\xspace}
\newcommand{\CE}{\texttt{E}\xspace} 
\title{Hypothesis Testing for Automated Community Detection in Networks}
\iftoggle{arxiv}{%
\author{Peter J. Bickel\\
\texttt{bickel@stat.berkeley.edu}\\
University of California,\\
Berkeley,USA.
\and
Purnamrita Sarkar\\
\texttt{psarkar@eecs.berkeley.edu}\\
University of California,\\
Berkeley, USA.}
}
{%
\author[Peter J. Bickel {\it et al.}]{Peter J. Bickel}
\address{University of California,
Berkeley,
USA.}
\email{bickel@stat.berkeley.edu}
\author[]{Purnamrita Sarkar}
\address{University of California,
Berkeley,
USA.}

}

\begin{document}
\iftoggle{arxiv}{%
\maketitle
}{}
\begin{abstract}
 Community detection in networks is a key exploratory tool with applications in a diverse set of areas, ranging from finding communities in social and biological networks to identifying link farms in the World Wide Web. The problem of finding communities or clusters in a network has received much attention from statistics, physics and computer science. However, most clustering algorithms assume knowledge of the number of clusters $k$.  In this paper we propose to automatically determine  $k$ in a graph generated from a \sbm.  Our main contribution is twofold; first, we theoretically establish the limiting distribution of the principal eigenvalue of the suitably centered and scaled adjacency matrix, and use that distribution for our hypothesis test. Secondly, we use this test to design a recursive bipartitioning algorithm. 
 Using quantifiable classification tasks on real world networks with ground truth, we show that our algorithm outperforms existing probabilistic models for learning overlapping clusters, and on unlabeled networks, we show that we uncover nested community structure.
\end{abstract}

%

\section{Introduction}
Network structured data can be found in many real world problems. Facebook is an undirected network of entities where edges are formed by who-knows-whom. The World Wide Web is a giant directed network with webpages as nodes and hyperlinks as edges. Finding community structure in network data is a key ingredient in many graph mining problems. For example, viral marketing targets tightly knit groups in social networks to increase popularity of a brand of product. There are many clustering algorithms in computer science and statistics literature. However, most of them suffer from a common issue: one has to assume that the number of clusters $k$ is known apriori. 

For labeled data, a common approach for learning $k$ is cross validating using held out data. However cross validation has two problems: it requires a lot of computation, and for sparse graphs it is sub-optimal to leave out data. In this paper we address this problem via a hypothesis testing framework based on random matrix theory.  This framework also naturally leads to a recursive bipartitioning algorithm, which leads to a hierarchical clustering of the data. 

For genetic data,~\citet{Patterson_populationstructure} show how to combine Principal Components Analysis with random matrix theory to discover if the data has cluster structure. This work uses existing results on the limit distribution of the largest eigenvalue of large random covariance matrices. 

In standard machine learning literature where datapoints are represented by real-valued features, ~\citet{pelleg-xmeans} jointly optimize over the set of cluster locations and number of cluster centers in \texttt{kmeans} to maximize the Bayesian Information Criterion (BIC).~\citet{hamerly-gmeans} propose a hierarchical clustering algorithm based on the Anderson-Darling statistic which tests if the data assigned to a cluster comes from a gaussian distribution.

For network clustering, finding the number of clusters automatically via a series of hypothesis tests has been proposed by~\citet{levina_pnas}. The authors present a label switching algorithm for extracting tight clusters from a graph sequentially. The null hypothesis is that there is no cluster structure. As pointed out by them, it is hard to define a null model. Possible candidates for null models are \er graphs, degree corrected block-models etc. The authors point out that for test statistics whose distributions under the null are hard to determine analytically, one can easily do a parametric bootstrap step to estimate the distribution. However this introduces a significant computational overhead for large graphs, since the bootstrap has to be carried out for each community extraction.

We focus on the problem of finding the number of clusters in a graph generated from a \sbm, which is a widely used model for generating labeled graphs~\citep{Holland1983}. Our null hypothesis is that there is only one cluster, i.e. the network is generated from a \er $G_{n,p}$ graph. Existing literature~\citep{necessary-sufficient-tw} can be used to show that the largest eigenvalue of the suitably scaled and centered adjacency matrix asymptotically has the Tracy-Widom distribution. Using recent theoretical results from random matrix theory, we show that this limit also holds when the probability of an edge $p$ is unknown, and the centering and scaling are done using an estimate of $p$.

We would like to emphasize that our theory holds for $p$ constant w.r.t $n$, i.e.\ the \textit{dense} asymptotic regime where the average degree is growing linearly with $n$. We are currently investigating the behavior of the largest eigenvalue when $p$ decays as $n\rightarrow\infty$. Experimentally we show how to obtain Bartlett type corrections~\citep{bartlett37} for our test statistic when the graph is small or sparse, i.e. the asymptotic behavior has not been reached. On quantifiable classification tasks on real world networks with ground truth, our method outperforms~\citet{McALes12}'s algorithm which has been shown to perform better than known methods for obtaining overlapping clusters in networks. Further, we show that our recursive bipartitioning algorithm gives a multiscale view of smaller communities with different densities nested inside bigger ones.

Finally, we conjecture that the second largest eigenvalue of the normalized Laplacian matrix also has a Tracy-Widom distribution in the limit. We are currently working on a proof.

\section{Preliminaries and Proposed Method}
Before presenting our main result, we introduce some notation and definitions.

\paragraph{\textsc{Stochastic Blockmodels:}}

For our theoretical results we focus on community detection in graphs generated from Stochastic Blockmodels. Informally, a \sbm with $k$ classes assigns latent cluster memberships to every node in a graph. Each pair of nodes with identical cluster memberships for the endpoints have identical probability of linkage, thus leading to stochastic equivalence. Let $Z$ denote a $n\times k$ binary matrix where each row has exactly one ``1'' and the $i^{th}$ column has $n_i$ ``1'''s; i.e. the $i^{th}$ class has $n_i$ nodes with $\sum_i n_i=n$. For this paper, we will assume that $Z$ is fixed and unknown. By definition there are no self loops. Thus, the conditional expectation of the adjacency matrix of a network generated from a \sbm  is given by
\begin{align}
\label{eq:sbmdef}
E[A|Z]=ZBZ^T-\diag(ZBZ^T),
\end{align}
where $\diag(M)$ is a diagonal matrix, with $\diag(M)_{ii}=M_{ii}$, $\forall i$. $A$ is symmetric and the edges are independent Bernoulli trials.
Because of the stochastic equivalence, the subgraph induced by the nodes in the $i^{th}$ cluster is simply an \er graph.

Thus, deciding if a \sbm has $k$ or $k+1$ blocks can be thought of as inductively deciding whether there is one block or two. In essence we develop a hypothesis test to determine if a graph is generated from an \er model with matching link probability or not. First we discuss some known properties of \er graphs. Throughout this paper we assume that the edge probabilities are constant, i.e. the average degree is growing as $n$.

\paragraph{\textsc{Properties of \er graphs:}}
Let $A$ denote the adjacency matrix of a \er(n,p) random graph, and let $P:=E[A]$. We will assume that there are no self loops and hence $A_{ii}=0,\forall i$.  Under the \er model, $P$ is defined as follows:
\begin{align}
\label{eq:P}
P=np\e\e^T-pI,
\end{align}
where $\e$ is length $n$ vector with $\e_i=1/\sqrt{n}$, $\forall i$, and $I$ is the $n\times n$ identity matrix.
We also introduce the following normalized matrices.
\begin{align}
\label{eq:def-mat}
\n{A}:=\frac{A-P}{\sqrt{(n-1)p(1-p)}} 
\end{align}
The  eigenvalues of $\n{A}$ are denoted by $\lambda_1\geq\lambda_2\geq\dots\geq \lambda_n$. 
Let us also define the density of the semi-circle law. In particular we have,
\begin{defn}
Let $\rho_{sc}$ denote the density of the semicircle law, defined as follows:
\begin{align}\label{eq:semicircle}
\rho_{sc}(x):=\frac{1}{2\pi}\sqrt{(4-x^2)_+}\qquad \mbox{$x\in \mathbb{R}$}
\end{align}
\end{defn}

 For Wigner matrices with entries having a symmetric law, the limiting behavior of the empirical distribution of the eigenvalues was established by~\citet{WIG58}. This distribution converges weakly to the semicircle law defined in Equation~\ref{eq:semicircle}. Also,~\citet{tracy-widom94} prove that for Gaussian Orthogonal Ensembles (G.O.E), $\lambda_1$ and $\lambda_n$, after suitable shifting and scaling converge to the Tracy-Widom distribution with index one ($TW_1$).~\citet{Soshnikov99} proved that the above universal result at the edge of the spectrum also holds for more general distributions, provided the random variables have symmetric laws of distribution, all their moments are finite, and $E[\n{A}_{ij}^m]\leq (C m)^m$ for some constant $C$, and positive integers $m$. This shows that $n^{2/3}(\lambda_1-2)$ weakly converges to the limit distribution of G.O.E matrices, i.e. the Tracy-Widom law with index one, for $p=1/2$.


Recently \citet{erdos-rigidity} have removed the symmetry condition and established the edge universality result for general Wigner ensembles.
Further~\citet{necessary-sufficient-tw} show a necessary and sufficient condition for having the limiting Tracy-Widom law, which shows that  $n^{2/3}(\lambda_1-2)$ converges weakly to $TW_1$ too.
If we know the true $p$, it would be easy to frame a hypothesis test which accepts or rejects the null hypothesis that  a network is generated from an \er graph. First we will compute $\theta:=n^{2/3}(\lambda_1-2)$, and then estimate the p-value $P(X\geq \theta)$  from available tables of probabilities for the Tracy-Widom distribution. Now for a predefined significance level $\alpha$, we reject the null if the p-value falls below $\alpha$.

However, we do not know the true parameter $p$; we can only estimate it within $O_P(1/n)$ error by computing the proportion of pairs of nodes that forms an edge. Let us denote this estimate by $\ph$. Thus the matrix at hand is $A-\hat{P}$, where $\hat{P}$ is:
\begin{align}
\label{eq:Phat}
\hat{P}=n\ph\e\e^T-\ph I,
\end{align}
 In this paper we show that the extreme eigenvalues of this matrix also follow the $TW_1$ law after suitable shifting and scaling.

\begin{thm}\label{thm:TW}
Let
\begin{align}
\label{eq:A-emp}
\natwo:=\frac{A-\hat{P}}{\sqrt{(n-1)\ph(1-\ph)}}.
\end{align}
We have,
\begin{align}
n^{2/3}\bb{\lambda_1(\natwo) - 2}\stackrel{d}{\rightarrow}\text{TW}_1
\end{align}
where $\mbox{TW}_1$ denotes the Tracy-Widom law with index one. This is also the limiting law of the largest eigenvalue of Gaussian Orthogonal Ensembles.
\end{thm}

Further, it is necessary to see that the above statistic does not have the Tracy-Widom distribution when $A$ is generated from a \sbm with $k>1$ blocks. We show that,  the statistic goes to infinity if $A$ is generated from a \sbm, as long as the class probability matrix $B$ is diagonally dominant. The diagonally dominant condition leads to clusters with more edges within than those across. A similar condition can be found in~\citet{levina_pnas} for proving asymptotic consistency of the extraction algorithm for Stochastic blockmodels with $k=2$. Further,~\citet{bickel2009nonparametric} also note that for $k=2$, the Newman-Girvan modularity is asymptotically consistent if this diagonal dominance holds.
We would like to note that this is only a sufficient condition used to simplify our proof. 
\begin{lem}\label{lem:sbm}
Let $A$ be generated from a \sbm with hidden class assignment matrix $Z$, and probability matrix $B$ (as in Equation~\ref{eq:sbmdef}) whose elements are constants w.r.t $n$.
If $\forall i, B_{ii}\geq \sum_{j\neq i} B_{ij} $, we have:
\begin{align}
\lambda_1(\natwo)\geq C_0\sqrt{n}
\end{align}
where $C_0$ is a deterministic positive constant independent of $n$.
\end{lem}

Given this result, we propose the following algorithm to find community structure in networks.
\newcommand*\Let[2]{\State #1 $\gets$ #2}
\begin{algorithm}[!htb]
  \caption{\label{alg:recursive}Recursive Bipartitioning of Networks Using Tracy-Widom Theory}
  \begin{algorithmic}[1]
    \Function{recursive bipartition}{$G,\alpha$}
      \Let {\ph}{$\dfrac{\sum_{i,j}A_{ij}}{n(n-1)}$}

      \Let {$\natwo$}{$\dfrac{A-\hat{P}}{\sqrt{(n-1)\ph(1-\ph)}}$}

      \Let{$\theta$}{$\lambda_1(\natwo)$}

      \Let{$pval$}{\textsc{HypothesisTest}($\lambda_1$,$n$,\ph)}
      \If{$pval<\alpha$}
      \Let {($G_1$,$G_2$)}{\textsc{bipartition}($G$)}\label{step:bipart}
       \State\textsc{recursive bipartition}($G_1,\alpha$)
       \State\textsc{recursive bipartition}($G_2,\alpha$)
       \EndIf
    \EndFunction
  \end{algorithmic}
\end{algorithm}
For the~\ref{step:bipart}$^{th}$ step in Algorithm~\ref{alg:recursive} we use the regularized version of Spectral Clustering introduced in~\citep{PL}. We want to emphasize that the choice of Spectral Clustering is orthogonal to the hypothesis test. One can use any other method for partitioning the graph.
\subsection{The Hypothesis Test}
Our empirical investigation shows that, while largest eigenvalues of G.O.E matrices converge to the Tracy-Widom distribution quite quickly, those of adjacency matrices do not. Moreover the convergence is even slower if $p$ is small, which is the case for sparse graphs. We elucidate this issue with some simulation experiments. We generate a thousand GOE matrices $M$, where $M_{ij}\sim N(0,1)$. In Figure~\ref{fig:TW_simulation} we plot the empirical density of $\lambda_1(M)/\sqrt{n}$ against the true Tracy-Widom density. In Figures~\ref{fig:TW_simulation}(A) and \ref{fig:TW_simulation}(B) we plot the GOE cases with $n$ equaling $50$ and $100$ respectively, whereas Figures~\ref{fig:TW_simulation}(C) and~\ref{fig:TW_simulation}(D) respectively show the \er cases with $n=50$,  $p=.5$ and $n=500$, $p=.5$.
\begin{figure}[!htb]
\caption{\label{fig:TW_simulation} We plot empirical distributions of largest eigenvalues against the limiting Tracy-Widom law. (A) GOE matrices with $n=50$, (B) GOE matrices with $n=500$. (C) \er graphs with $n=50,p=0.5$, (D) \er graphs with $n=500$ and $p=0.5$.}
 \begin{tabular}{@{\hspace{0em}}c@{\hspace{0em}}c}
\includegraphics[width=0.5\linewidth]{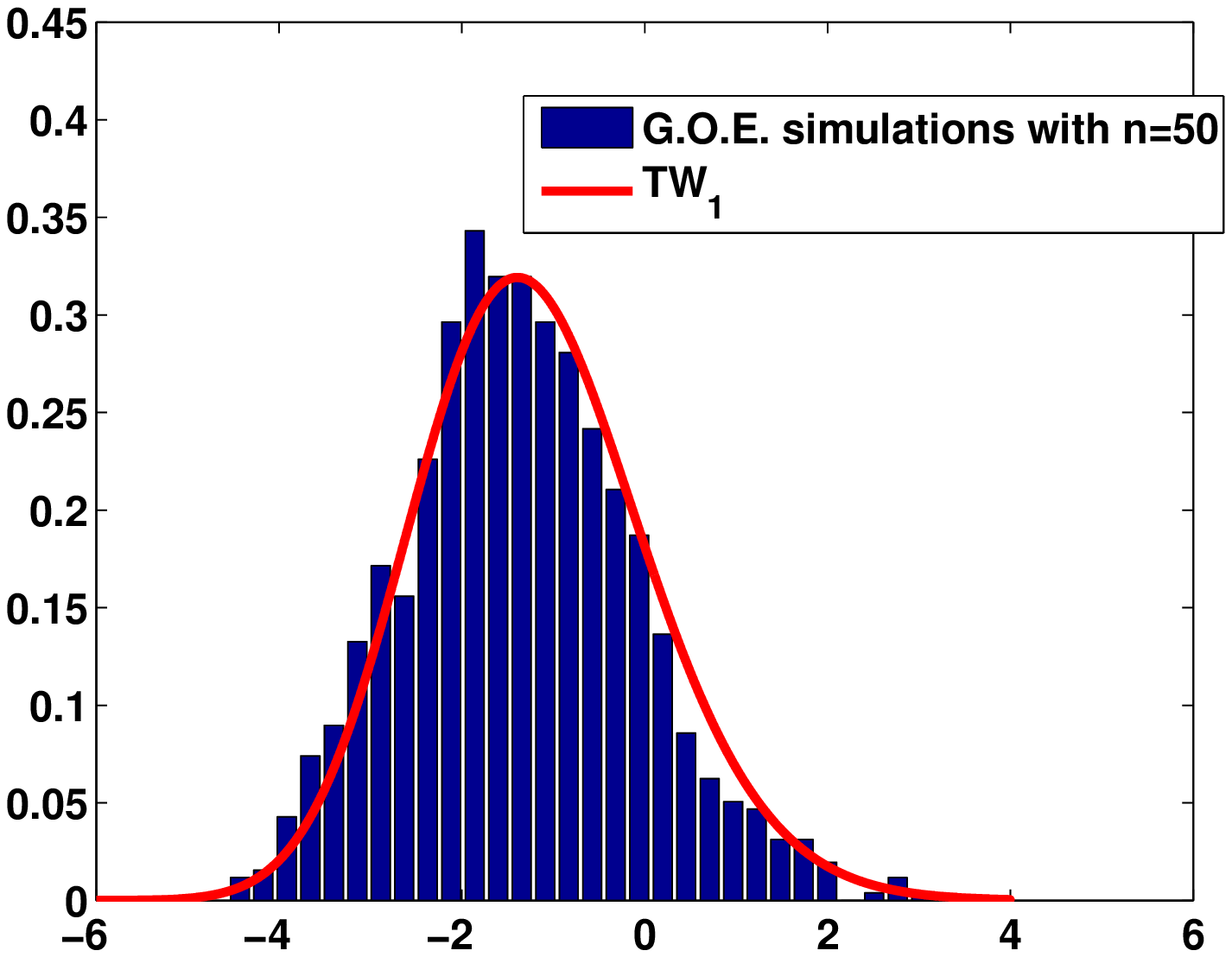}&
\includegraphics[width=0.5\linewidth]{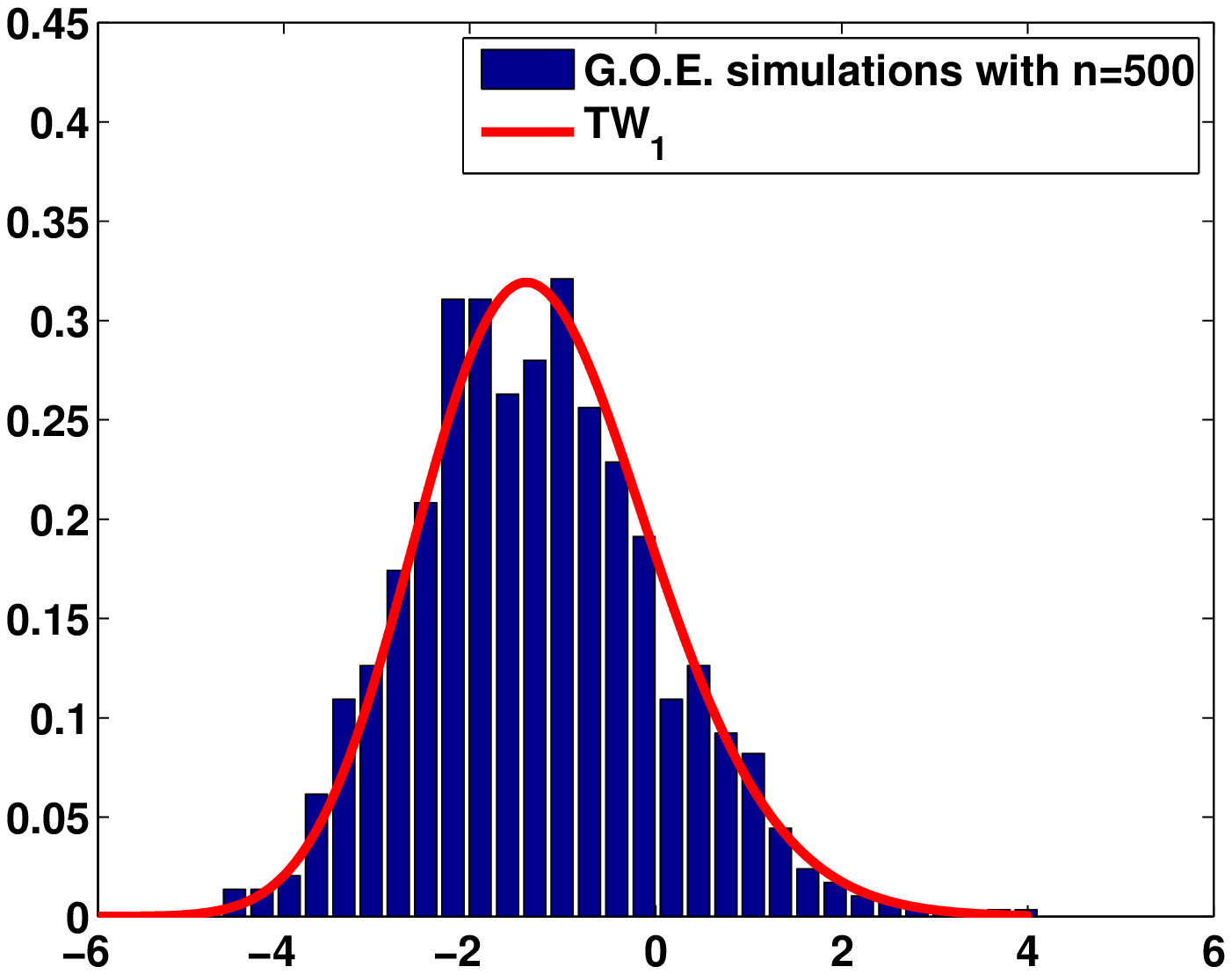}\\
(A)&(B)\\
\includegraphics[width=0.5\linewidth]{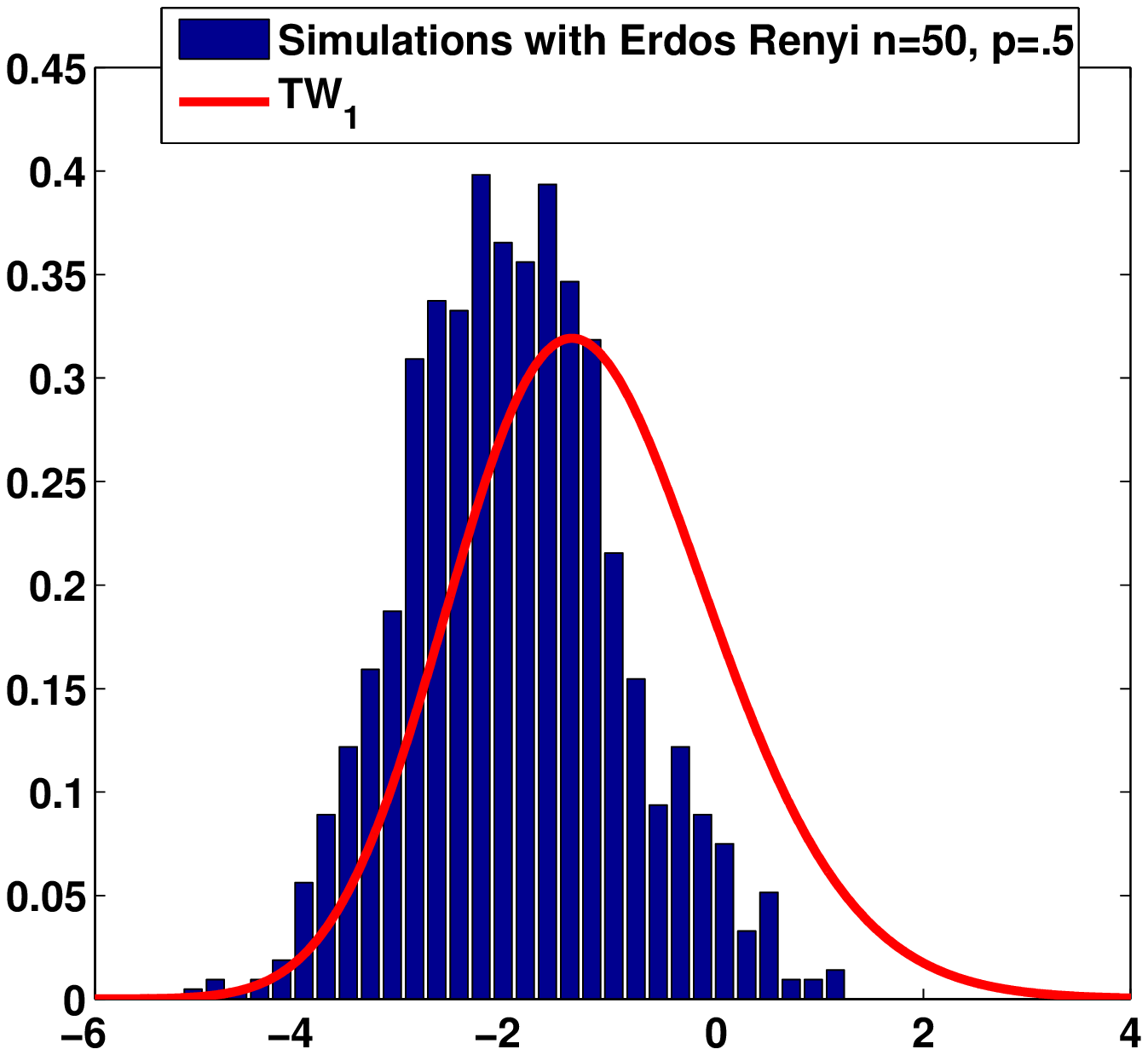}&
\includegraphics[width=0.56\linewidth]{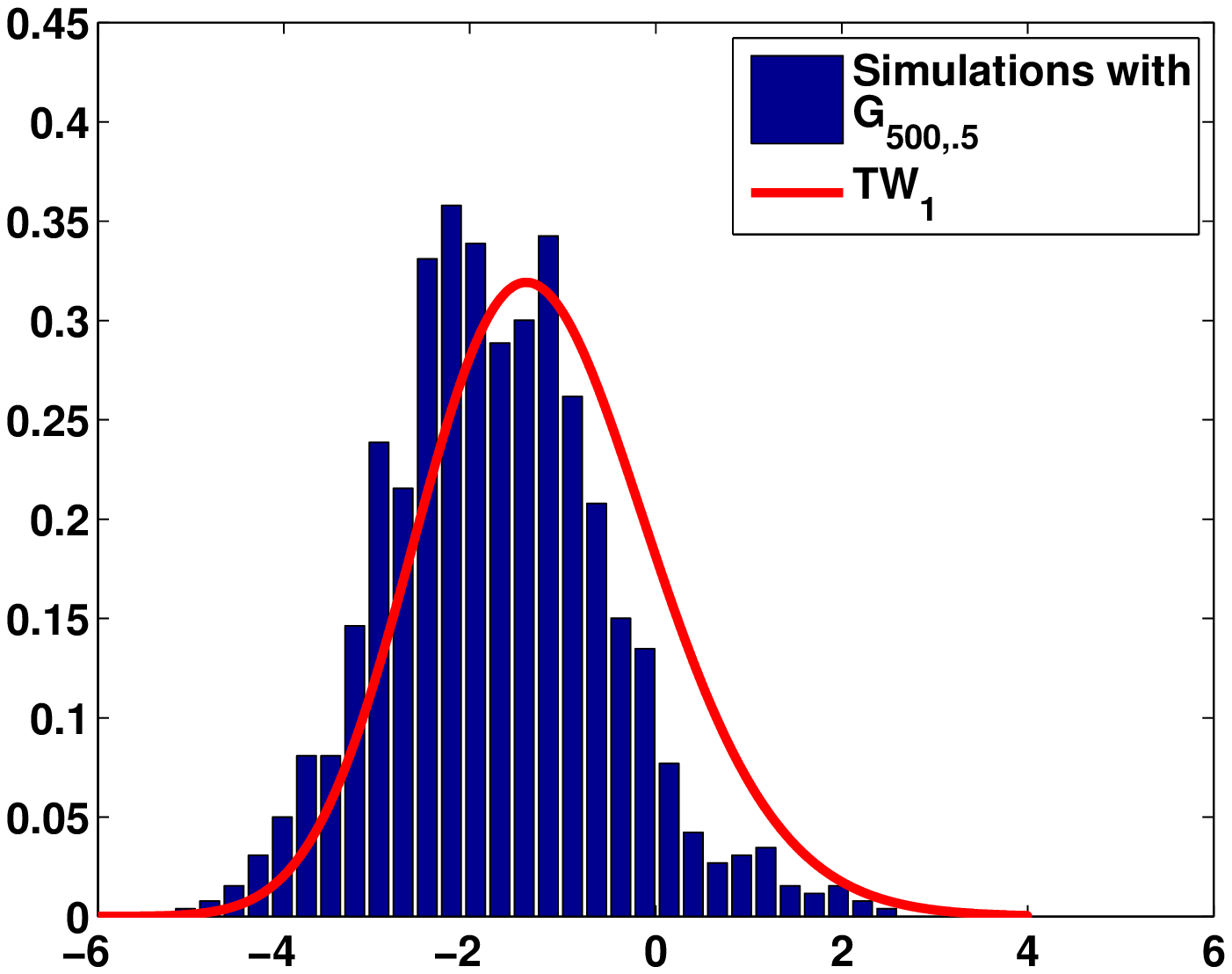}\\
 (C)&(D)
\end{tabular}
\end{figure}

This suggests that  computing the p-value using the empirical distribution of $\lambda_1$ generated using a parametric bootstrap step will be better than using the Tracy-Widom distribution. However, this will be computationally expensive, since it would have to be carried out at every level of the recursion in Algorithm~\ref{alg:recursive}. Instead we notice that if one can learn the shift and scale of the bootstrapped empirical distribution, it can be well approximated by the limiting $TW_1$ law. Hence we propose to do a few simulations to compute the mean and the variance of the distributions, and then shift and scale the test statistic to match the first two moments of the limiting $TW_1$ law.

\begin{figure}[!htb]
\caption{\label{fig:faster_estimation} We plot \textit{corrected} empirical distributions of largest eigenvalues computed using a thousand bootstrap replicates against the limiting Tracy-Widom law. On the leftmost panel we plot the original uncorrected empirical distribution. On the middle panel we present the corrected version with shift and scale estimated using 1000 samples, whereas on the rightmost panel, the shift and scale are estimated using 50 samples. (A), (B), and (C) are generated from $G_{50,0.5}$, whereas  (D), (E), and (F) are generated from $G_{200,0.05}$.}
 \begin{tabular}{@{\hspace{-1.2em}}c@{\hspace{-1.2em}}c@{\hspace{-1.2em}}c}
 \includegraphics[width=0.4\linewidth]{er_50_5}&
\includegraphics[width=0.4\linewidth]{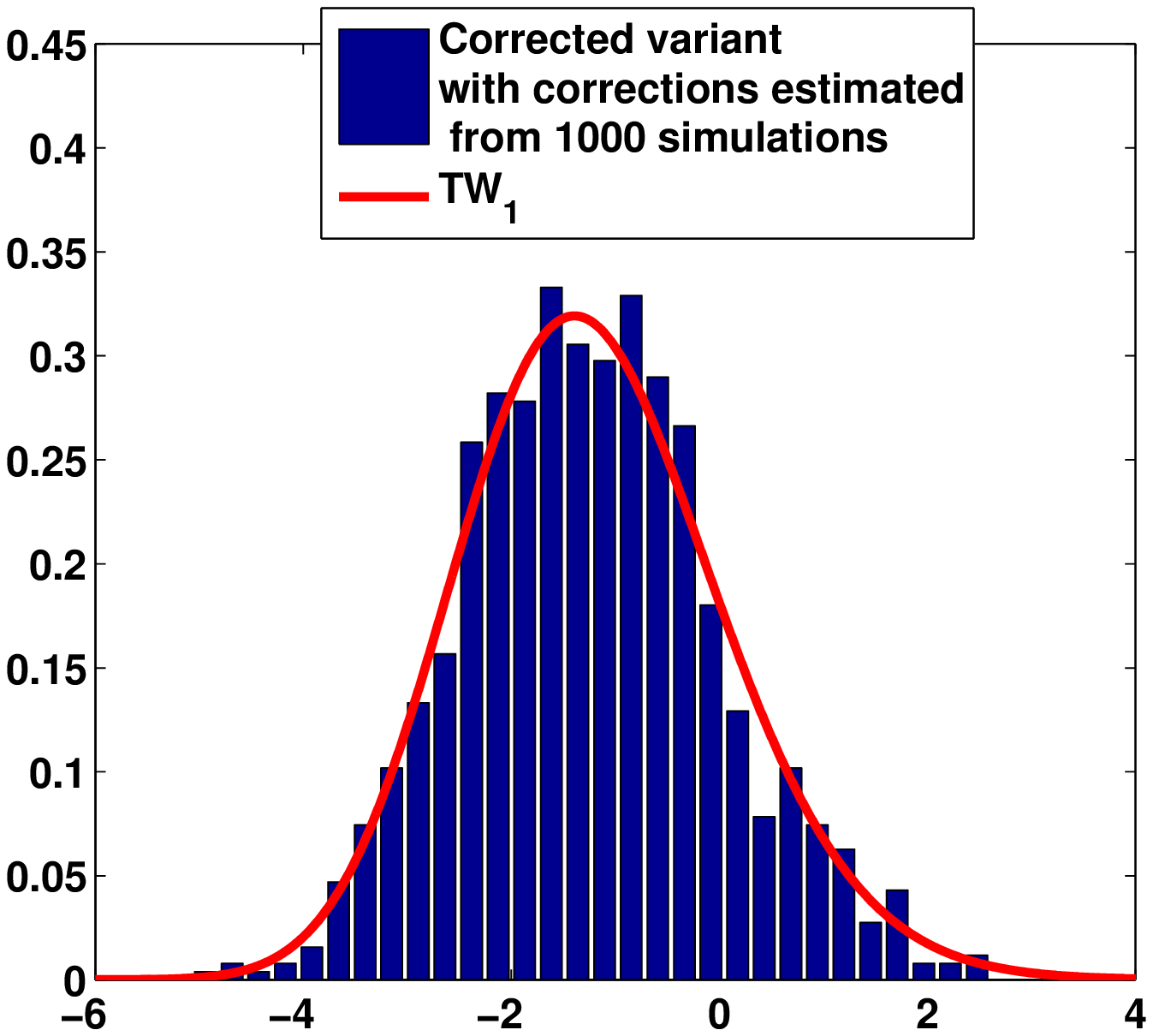}&
\includegraphics[width=0.4\linewidth]{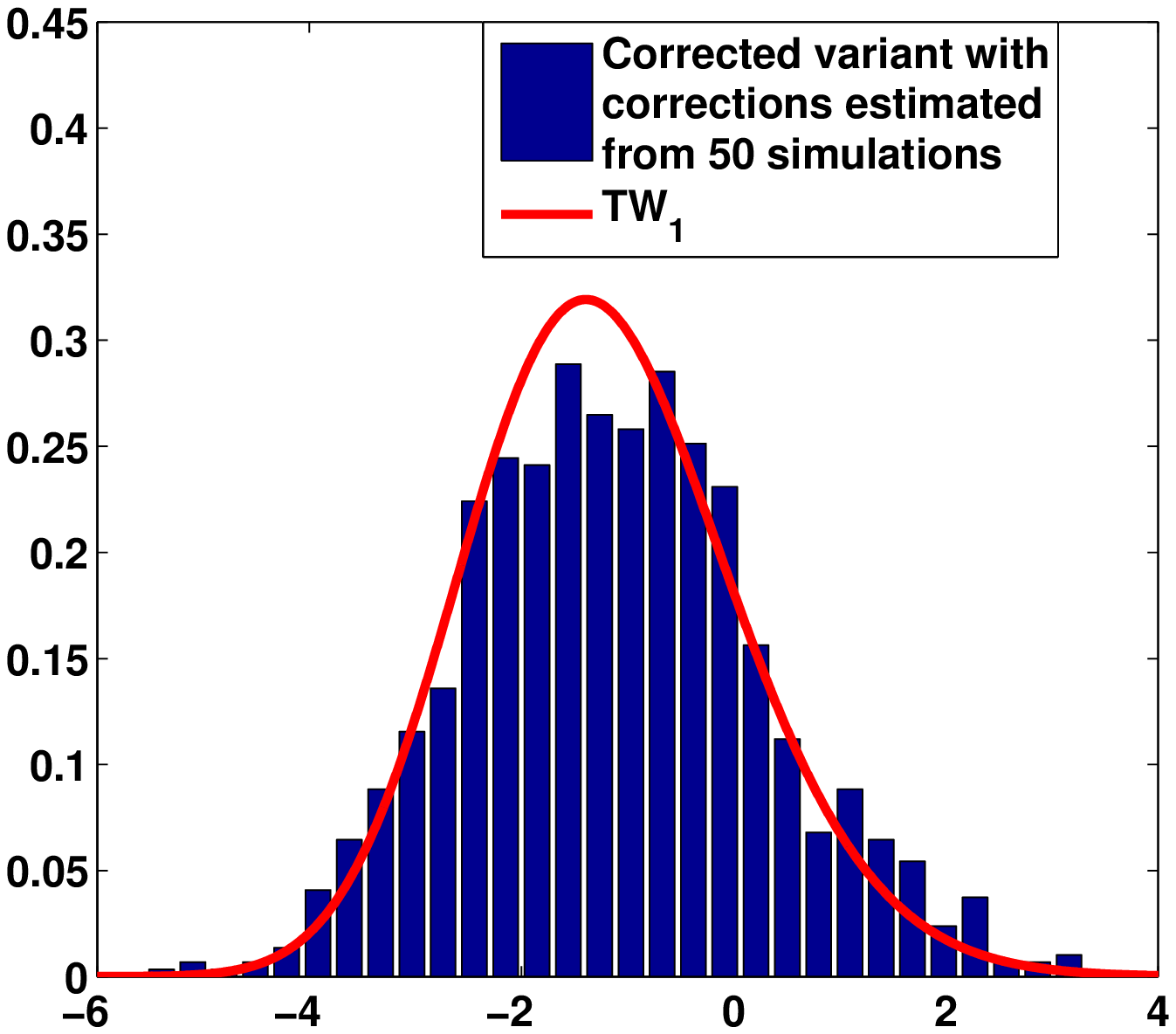}\\
(A)&(B)&(C)\\
\includegraphics[width=0.42\linewidth]{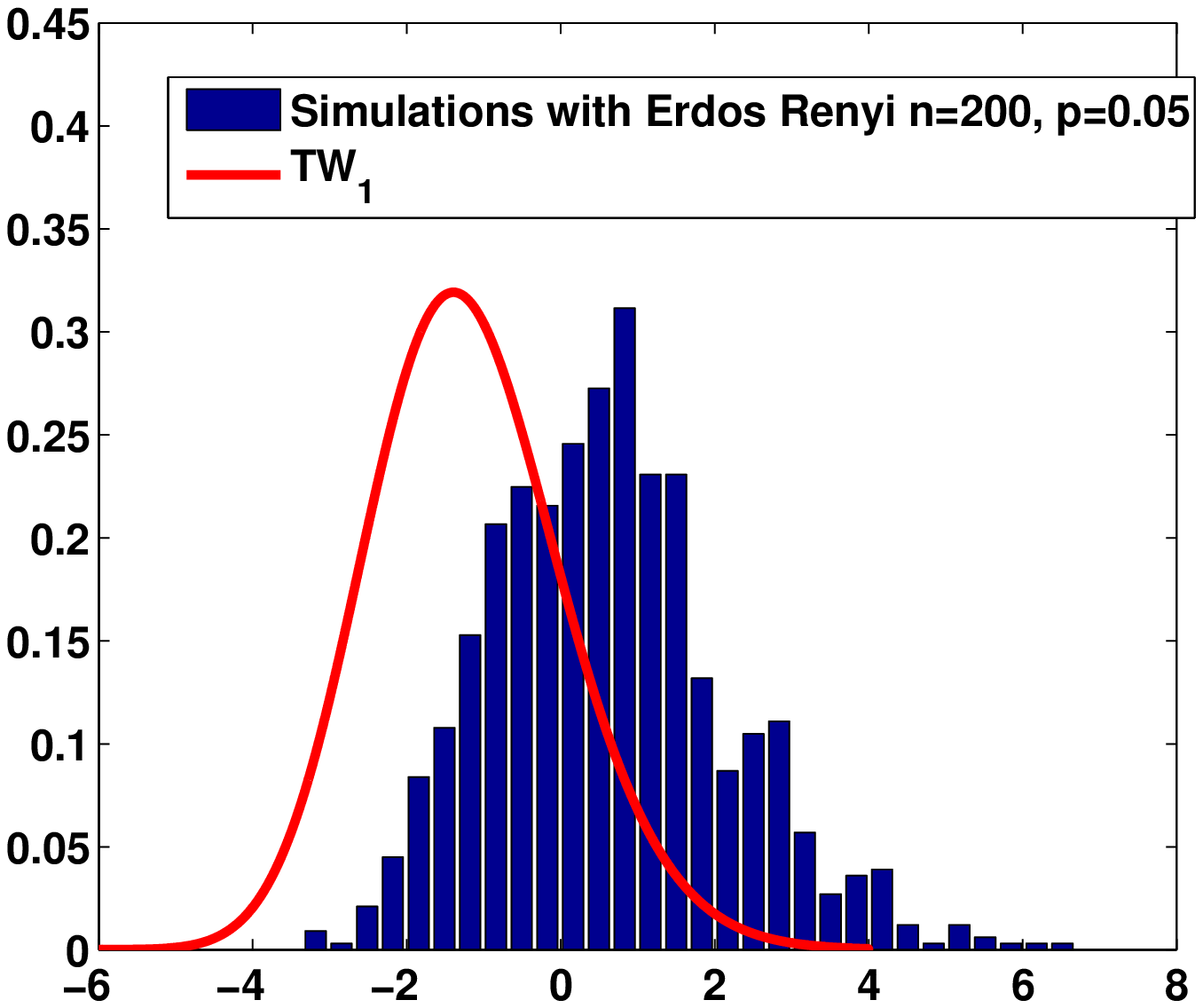}&
\includegraphics[width=0.4\linewidth]{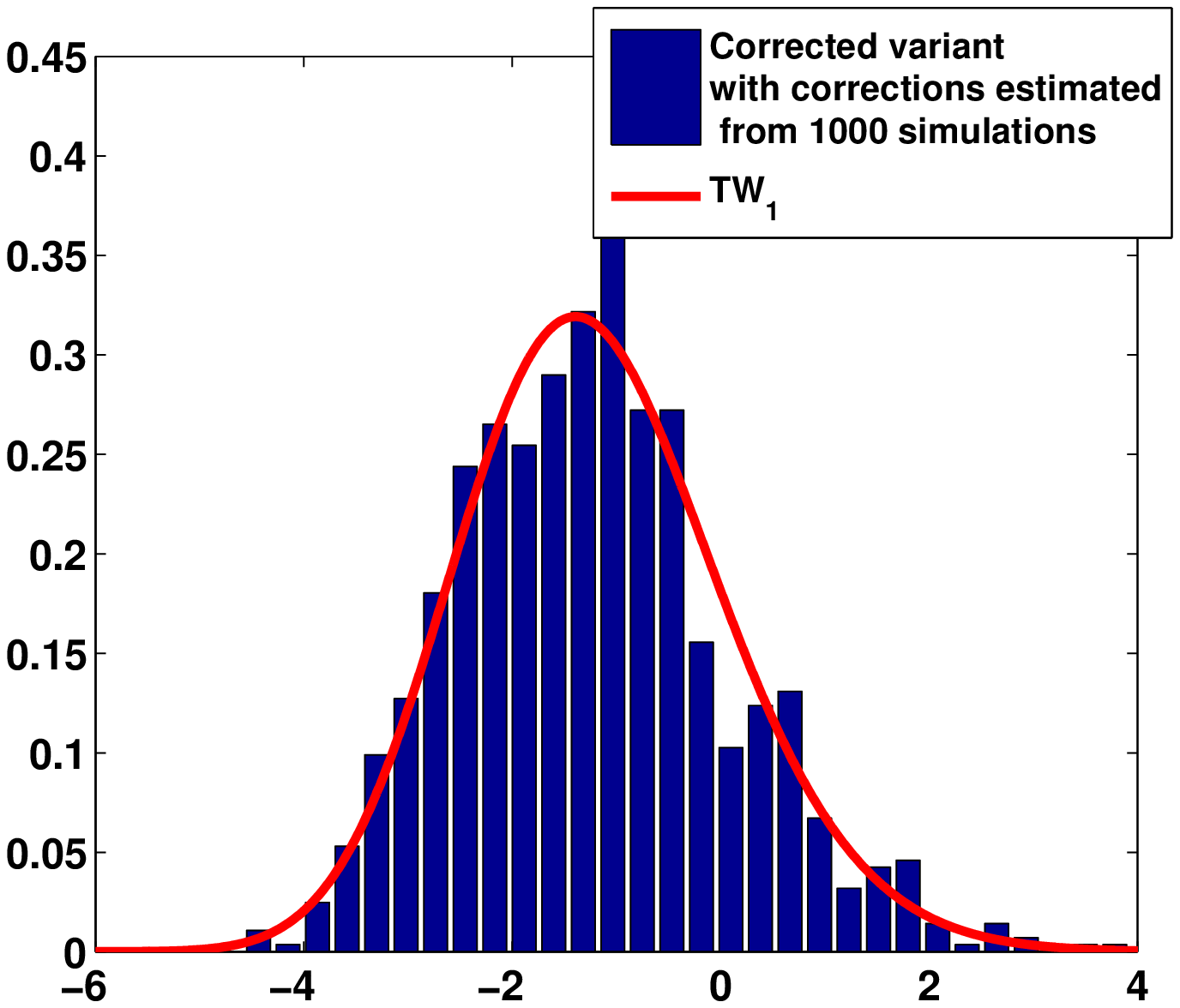}&
\includegraphics[width=0.4\linewidth]{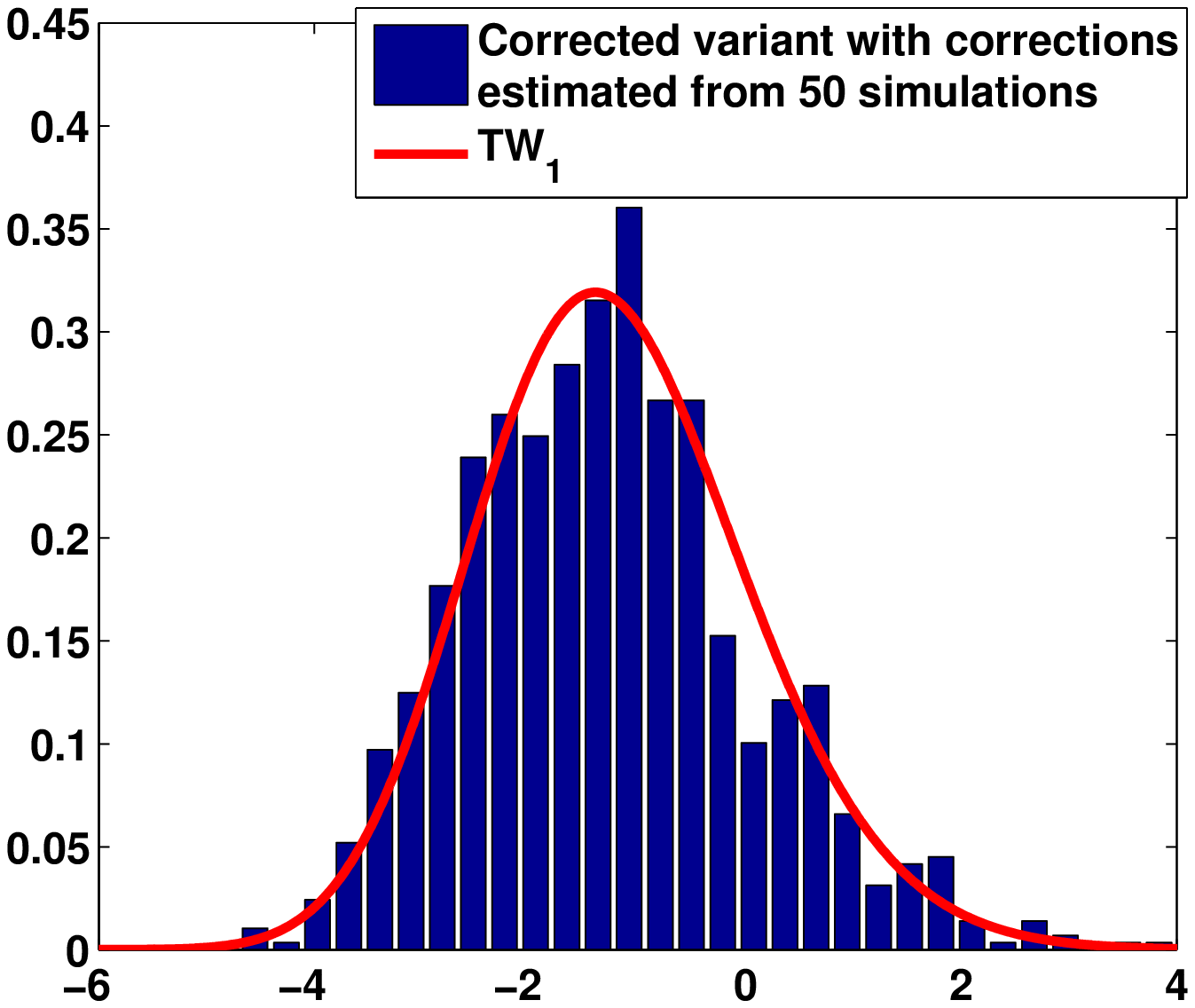}\\
(D)&(E)&(F)\\
\end{tabular}
\end{figure}

In Figure~\ref{fig:faster_estimation} we plot the empirical distribution of a thousand bootstrap replicates. The leftmost panels show how the empirical distribution of $\lambda_1$ differs from the limiting $TW_1$ law. In the middle panel we show the shifted and scaled version of this empirical distribution, where the mean and variance of the empirical distribution are estimated using a thousand samples drawn from the respective \er models. One can see that the middle panel is a much better fit to the Tracy-Widom distribution. Finally in the third panel we have the corrected empirical distribution where the mean and variance were estimated from fifty random samples. While this is not as good a fit as the middle panel, it is not much worse.

We would like to note that these corrections are akin to Bartlett type corrections~\citep{bartlett37} to likelihood ratio tests, which propose a family of limiting distributions, all scaled variants of the well-known chi-squared limit, and estimate the best fit using the data at hand. 
Now we present the hypothesis test formally. $E_{TW_1}[.]$, $\var_{TW_1}[.]$ and $P_{TW_1}[.]$ denote expectation, variance and probability of an event under the $TW_1$ law respectively.
\begin{algorithm}[!htb]
  \caption{Hypothesis Test
    \label{alg:hptest}}
  \begin{algorithmic}[1]
    \Function{HypothesisTest}{$\theta$,$n$,$\ph$}
    \Let{$\mu_{TW}$}{$E_{TW_1}[X]$}
        \Let{$\sigma_{TW}$}{$\sqrt{\var_{TW_1}[X]}$}
      \For{i=1:50}
      \Let{$A_i$}{\er(n,\ph)}
      \Let{$\theta_i$}{$\dfrac{\lambda_1(A-\hat{P})}{\sqrt{(n-1)\ph(1-\ph)}}$}
      \EndFor
      \Let{$\hat{\mu}_{n,\ph}$}{\textsc{mean}($\{\theta_i\}$)}
      \Let{$\hat{\sigma}_{n,\ph}$}{\textsc{standard deviation}($\{\theta_i\}$)}
      \Let{$\theta'$}{$\mu_{TW}+\left(\dfrac{\theta-\hat{\mu}_{n,\ph}}{\hat{\sigma}_{n,\ph}}\right)\sigma_{TW}$}
      \Let{pval}{$P_{TW_1}(X>\theta')$}
    \EndFunction
  \end{algorithmic}
\end{algorithm}

\paragraph{\textsc{Relationship to~\citet{levina_pnas}}.}
We conclude this section with a brief discussion of the similarities and differences of our work with the method in~\citep{levina_pnas}. The main difference is that their paper is focussed on finding and extracting communities which maximize a ratio-cut type criterion. We on the other hand do not prescribe a clustering algorithm. The clustering step in Algorithm~\ref{alg:recursive} is not tied to our hypothesis test and can easily be replaced by their community extraction algorithm. Computationally, our hypothesis testing step is faster, because we avoid the expensive parametric bootstrap to estimate the distribution of their statistic. This is possible because the limiting distribution is provably Tracy-Widom, and small sample corrections can be made cheaply by generating fewer bootstrap samples.  Finally, a superficial difference is that the  authors do a sequential extraction; the hypothesis test is applied sequentially on the complement of the communities extracted so far. We on the other hand, find the communities recursively, thus leading to a natural hierarchical clustering.
Thus if there are nested community structure inside an extracted community, this sequential strategy would miss that. We also demonstrate this in our simulated experiments.

We conclude this subsection with a remark on alternative hypothesis tests. In the context of a \sbm, one can use simpler statistics which exploit the i.i.d. structure of edges in each block of the network. For example $y_i:=\sqrt{n}((\lsum_j (A_{ij}-d_i/(n-1))^2/n)-\ph(1-\ph))$, $i=1,\dots,n$, have a limiting mean zero gaussian distribution under the \er model, and hence $\theta:=\max_i y_i$ should converge to a Gumbel distribution. Under a \sbm, $\theta$ will diverge to infinity because of the wrong centering. However, the hypothesis test with the principal eigenvalue worked much better in practice. The second eigenvalue of the Laplacian behaved similarly to our test. We would also like to point out that~\citet{ER2-erdos} show that the second largest eigenvalue of $A$ (with self loops), suitably centered and scaled, converges to the $TW_1$ law. This probably can also be used to design a hypothesis test by adjusting their proof technique. We would also like to note that it may be possible to design
hypothesis tests that use the limiting behaviors of the number of
paths or cycles in \er graphs using limiting results from~\citet{phase-transition-rand-strctures}.

\subsection{Conjecture on the Normalized Laplacian matrix}
We conclude this section with a conjecture on the second largest eigenvalue of the graph Laplacian matrix. Like~\citet{rohe_chatterji_yu} we will adopt the following definition of the Laplacian. Let $L:=D^{-1/2}AD^{-1/2}$, where $D$ is the diagonal matrix of degrees, i.e. $D_{ii}=\sum_j A_{ij}$.
\begin{conj}
Let $A$ be the adjacency matrix of an \er $G_{n,p}$ graph, and let $L:=D^{-1/2}AD^{-1/2}$ denote the normalized Laplacian. If $p$ is a fixed constant w.r.t $n$, we have:
\begin{align*}
n^{2/3}\bb{\sqrt{\frac{n\ph}{1-\ph}}(\lambda_2(L)+1/n)-2}\stackrel{d}{\rightarrow} TW_1
\end{align*}
\end{conj}
The intuition behind this conjecture is that the second eigenvalue of $L$ can be thought of as  $\max_{y\perp \xv_1}y^T(L-\xv_1\xv_1^T)y$. Here $\xv_1$ is the eigenvector corresponding to eigenvalue one. It is easy to show that $\xv_1(i)= \sqrt{D_{ii}/E}$, where $E=\sum_i D_{ii}$. In fact, using simple Chernoff-bound type arguments, one can show that $D_{ii}/E$ concentrates around $1/n$. On the other hand, elements of $L_{ij}$ can be approximated in the first order by $A_{ij}/n\hat{p}$. Thus the difference $L-\xv_1\xv_1^T$, can be approximated by $(A-n\ph\e\e^T)/n\ph$. We show that the largest eigenvalue of this matrix has the limiting $TW_1$ distribution, after suitable scaling and shifting.  Moreover, eigenvectors of $A-n\ph\e\e^T$ corresponding have been shown~\citep{knowles-Yin-general} to be almost orthogonal to the all ones vector, which is a close approximation of $\xv_1$ for \er graphs. We are currently working on proving this conjecture. 

\begin{figure}
\caption{\label{fig:laplacian-simulations}Simulations using the statistic obtained from the Normalized Laplacian matrix. (A) n=500, p=0.5 (B) n=50, p=0.5, (C) n=200, p=0.05}
\begin{tabular}{@{\hspace{-1em}}c@{\hspace{-1em}}c@{\hspace{-1em}}c}
\includegraphics[width=0.4\linewidth]{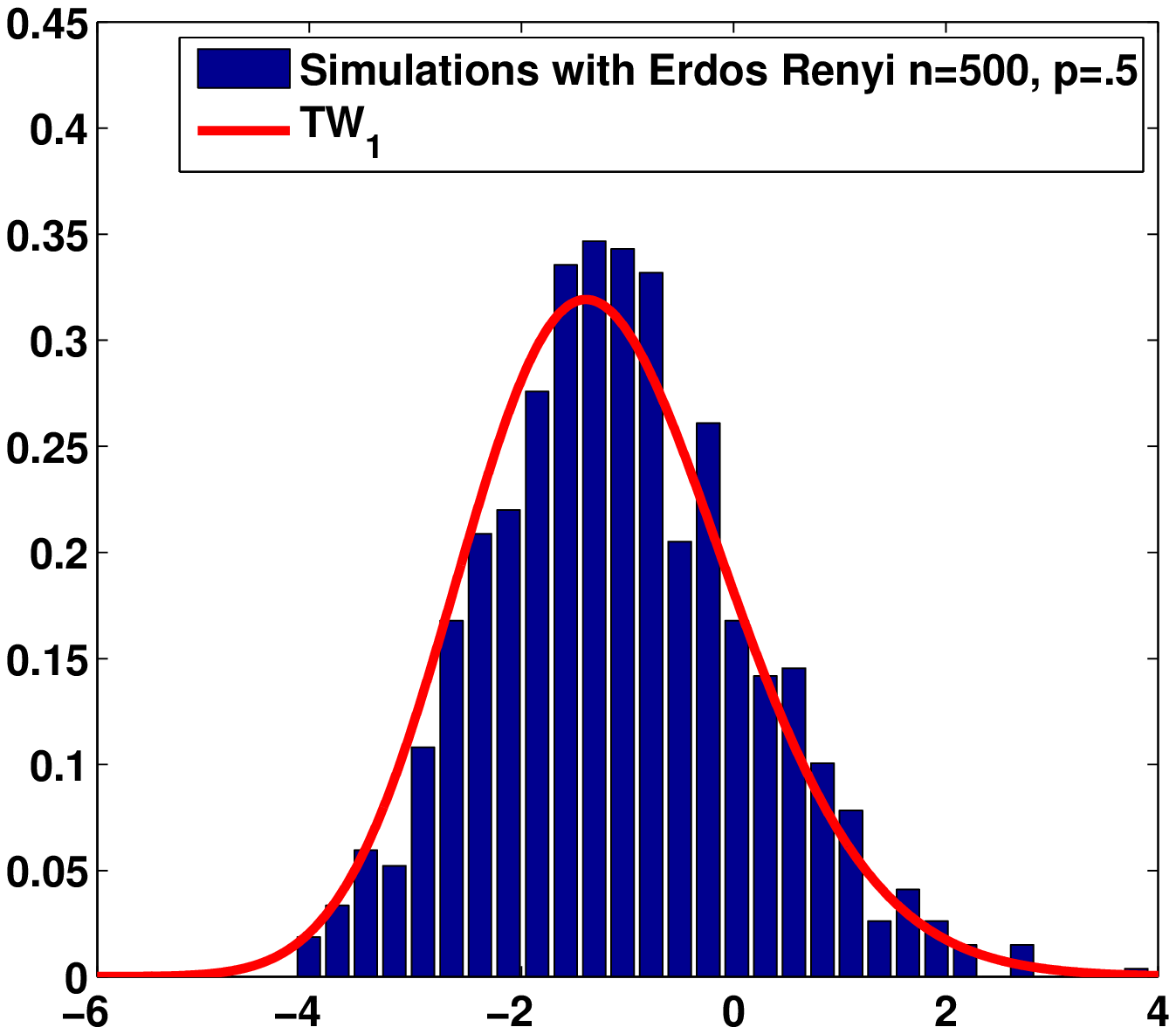}&\includegraphics[width=0.4\linewidth]{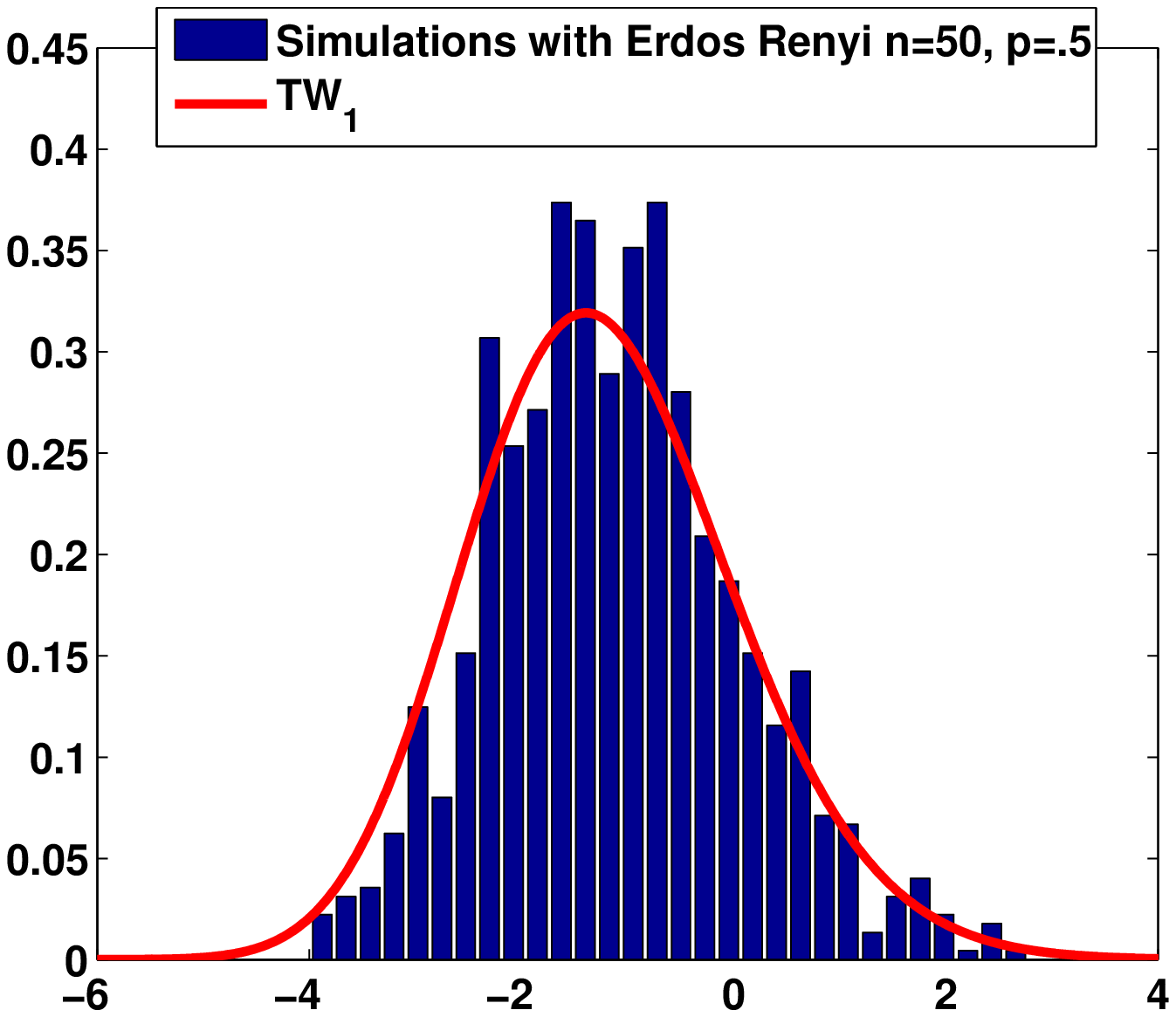}&\includegraphics[width=0.4\linewidth]{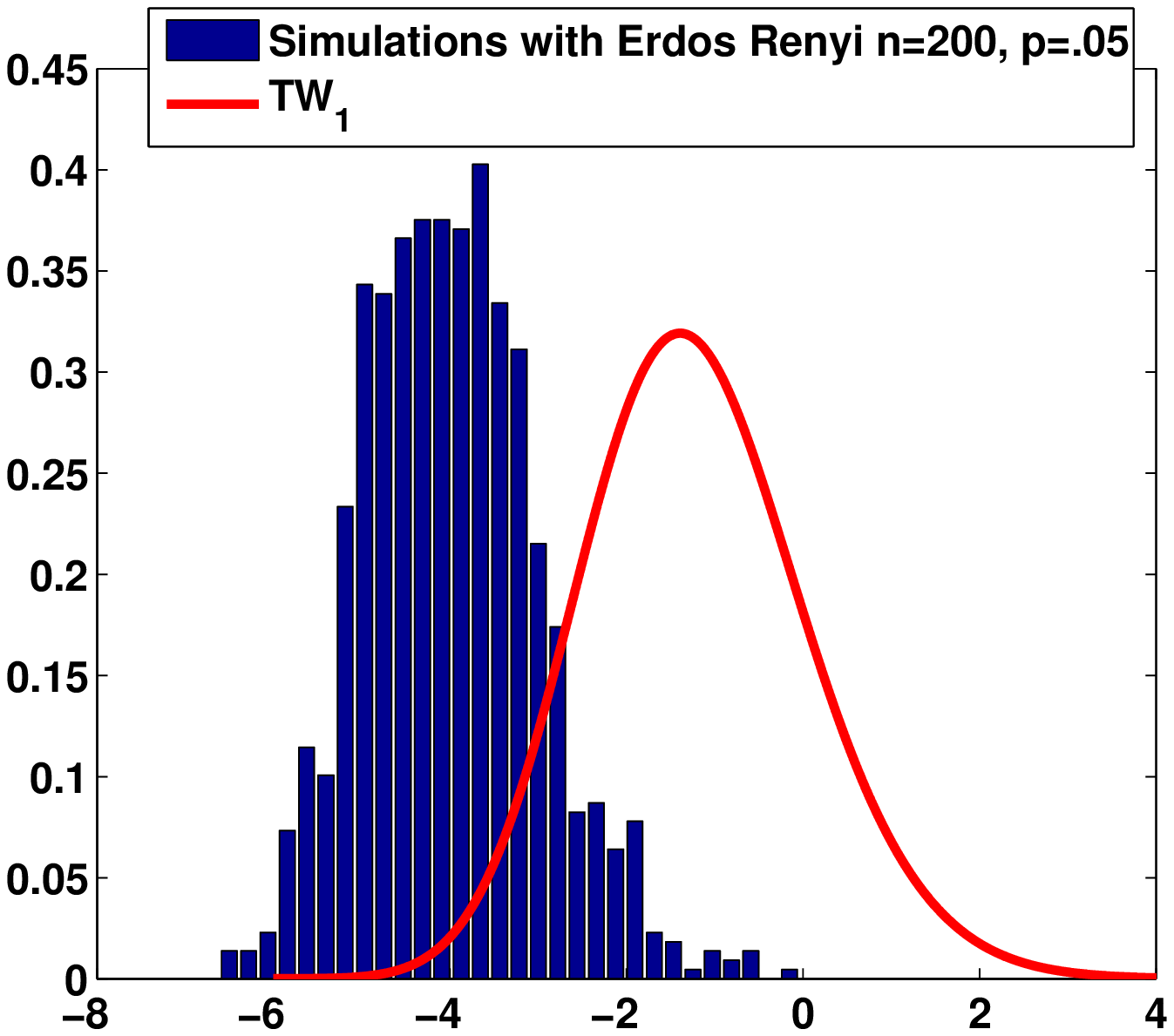}\\
(A)&(B)&(C)\\
\end{tabular}
\end{figure}

Figure~\ref{fig:laplacian-simulations} shows the fit of the statistic obtained from $L$ with the $TW_1$ law. Figures~\ref{fig:laplacian-simulations}(A) and (B) show that for dense graphs the statistic using $L$ converges to the limiting law faster than the corresponding statistic using the adjacency matrix $A$.  However, Figure~\ref{fig:laplacian-simulations}(C) shows that for sparse graphs, convergence is slow, similar to the adjacency matrix case. Experimentally, we saw that the same correction using the data leads to better fit for this case as well.
\section{Proof of Main Result}
In this section we will present the proof of Theorem~\ref{thm:TW}. Our proof uses the following machinery developed in random matrix theory in recent years.
Recently~\citet{erdos-rigidity} have proved that eigenvalues of general symmetric Wigner ensembles follow the local semicircle law. In particular, in the bulk, it is possible to estimate the empirical eigenvalue density using the semicircle law.
\begin{res}[Equation 2.26 in~\citet{erdos-rigidity}]
\label{res:local-semicircle}
Let $\lambda_1\geq\lambda_2\geq \dots\lambda_n$ be the eigenvalues of $\n{A}$ (Equation~\ref{eq:def-mat}). Also let $p$ be a constant w.r.t $n$. We define the following empirical and population quantities:
\begin{align*}
\mathcal{N}(a,b)=|i:a<\lambda_i\leq b|\qquad{\mathcal{N}_{sc}}(a,b):=n\int_{a}^{b}\rho_{sc}(x)dx,
\end{align*}
where $\rho_{sc}$ is defined in Equation~\ref{eq:semicircle}.
There exists positive constants $A_0>1$, $C$, $c$ and $d<1$, such that for any $L$ with
\begin{align}
\label{eq:L}
A_0\log\log n\leq L\leq \log(10n)/\log\log n,
\end{align}
we have:
$$P\bb{\sup_{|E|\leq 5}|\mathcal{N}(-\infty,E)-\mathcal{N}_{sc}(-\infty,E)|\geq (\log n)^L}\leq C\exp[-c(\log n)^{-dL}]$$
for sufficiently large $n$.
\end{res}
We will use the above result to obtain a probabilistic upper bound on the local eigenvalue density.
We note that, for $|a|,|b|\leq 5$,
{\small
\begin{align}
&P\bb{|\mathcal{N}(a,b)-\mathcal{N}_{sc}(a,b)|\geq 2(\log n)^L}\nonumber\\
&\leq P\bb{|\mathcal{N}(-\infty,b)-\mathcal{N}_{sc}(-\infty,b)|\geq (\log n)^L}+P\bb{|\mathcal{N}(-\infty,a)-\mathcal{N}_{sc}(-\infty,a)|\geq (\log n)^L}\nonumber\\
&\leq 2C\exp[-c(\log n)^{-dL}]\label{eq:local-density-er1}
\end{align}
}

We will also use the result the following probabilistic upper bound on the largest absolute eigenvalue:
\begin{res}[Equation 2.22,~\citet{erdos-rigidity}]
There exists positive constants $A_0>1$, $C$, $c$ and $d<1$, such that for any $L$ satisfying Equation~\ref{eq:L} we have,
\label{res:maxeig}
$$P\bb{\max_{j=1,\dots,n}|\lambda_j|\geq 2+n^{-2/3}(\log n)^{9L}}\leq C\exp[-c(\log n)^{dL}]$$
\end{res}


First we will state the necessary and sufficient condition for the Tracy-Widom limit of the extreme eigenvalues of a generalized Wigner matrix.

\begin{res}[Theorem 1.2, ~\citep{necessary-sufficient-tw}]
\label{res:nec-suf-tw}
Define a symmetric Wigner matrix $H_n$ of size $n$ with
\begin{align}
H_{ij}=h_{ij}=\frac{x_{ij}}{\sqrt{n}},\mbox{  $1\leq i,j\leq n$}.
\end{align}
The upper triangular entries are independent real random variables with mean zero satisfying the following conditions:
    \vspace{-1em}
\begin{itemize}
\item The off diagonal entries $x_{ij}$ ($1\leq i<j\leq n$) are i.i.d random variables satisfying $E[x_{12}]=0$ and $E[x_{12}^2]=1$.
\item The diagonal entries $x_{ii}$, ($1\leq i\leq n$) are i.i.d. random variables satisfying $E[x_{11}]=0$ and $E[x_{11}^2]<\infty$.
\end{itemize}
Also consider the simple criterion:
\begin{align}
\label{eq:nec-suff-cond}
\lim_{s\rightarrow\infty}s^4P(|x_{12}\geq s|)=0
\end{align}
Then, the following holds:
\begin{itemize}
\item Sufficient condition: if condition~\ref{eq:nec-suff-cond} holds, then for any fixed $k$, the joint distribution function of $k$ rescaled largest eigenvalues $\lambda_1(H),\dots,\lambda_n(H)$,
    \begin{align}
    \label{eq:joint-eigval}
    P\bb{n^{2/3}(\lambda_1(H)-2)\leq s_1,\dots,n^{2/3}(\lambda_k(H)-2)\leq s_k}
    \end{align}
    has a limit as $n\rightarrow\infty$, which coincides with the GOE case, i.e. it weakly converges to the Tracy-Widom distribution.
    \vspace{-1em}
\item Necessary condition: if condition~\ref{eq:nec-suff-cond} does not hold, then the distribution function in Equation~\ref{eq:joint-eigval} does not converge to a Tracy-Widom distribution. Furthermore, we have:
    $\limsup_{n\rightarrow\infty}P(\lambda_1(H)\geq 3)>0.$
\end{itemize}
\end{res}

Our definition of $\n{A}$ was designed to match the conditions required for Result~\ref{res:local-semicircle}. However, it is easy to see matrix $H:=\sqrt{(n-1)/n} \n{A}$ matches the setting in Result~\ref{res:nec-suf-tw}. Because $x_{12}$ in this case is a centered Bernoulli, condition~\ref{eq:nec-suff-cond} trivially holds. Thus we have  $\lambda_1(H)=2+n^{-2/3}TW_{1}+o_P(n^{-2/3})$. However, the $\sqrt{n/(n-1)}$ factor scales the eigenvalues by $1+O(1/n)$, which does not mask the $n^{-2/3}$ coefficient on the Tracy-Widom law. Thus we also have:
\begin{align}
\label{eq:TW}
\lambda_1=2+n^{-2/3}TW_1+o_P(n^{-2/3})
\end{align}

 \citet{knowles-Yin-general} prove the following isotropic delocalization result for eigenvectors of generalized Wigner matrices. 
  we define the $\n{O}_P(\zeta)$ notation (denoted by $\prec$ in the original paper), for a sequence of random variables which are bounded in probability by a positive random variable $\zeta$ up-to small powers of $n$. 
\begin{defn}
\label{def:nO}
We define $X_n=\n{O}_P(\zeta)$, Iff
$$  \mbox{$\forall$ (small) $\epsilon$, and (large) $D>0$, $P(|X_n|/n^\epsilon\geq \zeta)<n^{-D}$, $\forall n\geq N_0(\epsilon,D)$.}$$
\end{defn}
\begin{res}[Theorem 2.16,~\citet{knowles-Yin-general}]
\label{res:eigvec-proj}
Let $H=H^T$ be a generalized real symmetric Wigner matrix whose elements are independent random variables with the following conditions:
$E H_{ij}=0$, $E[H_{ij}]^2=:s_{ij}$, with
\begin{align}
\label{eq:cond-knowles-general}
1/C\leq ns_{ij}\leq C \qquad \sum_j s_{ij}=1
\end{align}
for some constant $C>0$. All moments of the entries are finite in the sense that for all $p\in\mathbb{N}$, there exists a constant $C_p$ such that $E|\sqrt{n}H_{ij}|^p\leq C_p$.

Let $\vv_i(H)$ be the $i^{th}$ eigenvector of $H$ corresponding to the $i^{th}$ largest eigenvalue $\lambda_i(H)$. For any deterministic vector $\wv$, we have:
\begin{align}
|(\wv^T\vv_i(H))^2|=\n{O}_P(1/n)
\end{align}
uniformly for all $i=1,\dots,n$.
\end{res}

We want to note that, since we do not allow self loops, for $\n{A}$, $s_{ii}=0$ for all $i=1,\dots, n$. Hence the first half of condition~\ref{eq:cond-knowles-general} does not hold. In order to relax this condition, we note that this result is proven using the isotropic local semicircle law (Theorem 2.12 in~\citet{knowles-Yin-general}), which is a direct consequence of the local entry-wise semicircle law (Theorem 2.13 in the same). However the entry-wise semicircle law from recent work of~\citet{knowles-Yin-general2} (Theorem 2.3) applies to our setting, and by using this instead of Theorem 2.13 in the chain of arguments in~\citep{knowles-Yin-general}, we can apply Result~\ref{res:eigvec-proj} to eigenvectors of $\n{A}$. Let $\vv_i$ be the eigenvector of $\n{A}$ corresponding to its $i^{th}$ largest eigenvalue $\lambda_i$. Let $\e$ be the $1/\sqrt{n}(1,\dots,1)$ vector. We have:
\begin{align}
\label{eq:knowles-yin-application}
|(\e^T\vv_i)^2|=\n{O}_P(1/n)
\end{align}
uniformly for all $i=1,\dots,n$.


We will now present Weyl's Interlacing Inequality, which would be used heavily in our proof.
\begin{res}
\label{res:interlace}
Let $B_1$ be an $n\times n$ real symmetric matrix and $B_2=B_1+d\xv\xv^T$, where $d>0$ and $\xv\in \mathbb{R}^n$. Denoting the $i^{th}$ largest eigenvalue of matrix $(.)$ by $\lambda_i(.)$ we have:
\begin{align}
\lambda_n(B_1) \leq \lambda_n(B_2)\leq \lambda_{n-1}(B_1)\leq  \dots \leq \lambda_{2}(B_2)\leq \lambda_1(B_1)\leq \lambda_1(B_2) \label{eq:interlaceone}
\end{align}
An immediate corollary of this result is that for $d<0$,
\begin{align}
\lambda_n(B_2) \leq \lambda_n(B_1)\leq \lambda_{n-1}(B_2)\leq  \dots \leq \lambda_{2}(B_1)\leq \lambda_1(B_2)\leq \lambda_1(B_1)
\label{eq:interlacetwo}
\end{align}
\end{res}
Let $\ph:=\sum_{ij}A_{ij}/n(n-1)$, and let $\e$ denote the normalized $n\times 1$ vector of all ones. As in Equation~\ref{eq:Phat}, $\hat{P}$ is the empirical version of $P$ (Equation~\ref{eq:P}). 
\begin{lem}
\label{lem:tw1}
Let $\n{A}_1:=\n{A}+\frac{n(p-\ph)\e\e^T}{\sqrt{(n-1)p(1-p)}}$. Also let $\lambda_1\geq\lambda_2\geq \dots\lambda_n$ be the eigenvalues of \na and $\mu_1\geq\mu_2\geq \dots\geq \mu_n$ be the eigenvalues of \nao. If $p$ is a constant w.r.t $n$, we have:
$$|\mu_1-\lambda_1|=o_P(1/n)$$
\end{lem}

\begin{proof}
Let $\lambda_i,\vv_i$ be eigenvalues and eigenvectors of $\n{A}$, where $\lambda_i\geq \lambda_{i+1},i\in\{1,\dots,n-1\}$. Also, let $\mu_i,\uv_i$ be eigenvalues and eigenvectors of $\n{A}_1$, also arranged in decreasing order of $\mu$. 
Let $G:=(\n{A}-zI)^{-1}$ and $G_1:=(\n{A}_1-zI)^{-1}$ be the resolvents of $\n{A}$ and $\n{A}_1$. 
Let $c_n:=\sqrt{n}(\ph-p)/\sqrt{p(1-p)}\sqrt{n/(n-1)}$. We note that the matrices $\n{A}$ and $\n{A}_1$ differ by a random multiple of the all ones matrix.
\begin{align}
\label{eq:all-ones-dev}
\n{A}=\n{A}_1+c_n\e\e^T
\end{align}
The above equation also gives
\begin{align}
\label{eq:weyl-dev}
|\lambda_1-\mu_1|\leq |c_n| = O_P(1/\sqrt{n})
\end{align}

The above is true because \ph is the average of $n(n-1)/2$ i.i.d Bernoulli coins, and thus $|c_n|=O_P(1/\sqrt{n})$ for $p$ constant w.r.t $n$. However this error masks the $n^{-2/3}$ scale of the Tracy-Widom law. 


Equation~\ref{eq:all-ones-dev} also gives the following identity:
\begin{align}
&\e^T (G(z)-G_1(z))\e=-c_n (\e^T G(z)\e) (\e^T G_1(z)\e )\nonumber\\
&\frac{1}{\e^TG_1(z)\e}-\frac{1}{\e^TG(z)\e}=-c_n
\end{align}
Since $1/\e^T G_1(\mu_1)\e=0$, we have $\e^T G(\mu_1)\e=\frac{1}{c_n}$.
 Further, using Weyl's interlacing result~\ref{res:interlace} we see that the eigenvalues of $\na$ and $\nao$ interlace.
%
 Since $G(z)$'s eigenvalues and vectors are given by $1/(\lambda_i-z)$, and $\vv_i$ respectively, we have: 
$$\frac{1}{c_n}=\e^TG(\mu_1)\e=\sum_i \frac{(\e^T\vv_i)^2}{\lambda_i-\mu_1}$$

Since the interlacing of eigenvalues depend on the sign of $c_n$, we will now do a case by case analysis.
\paragraph{Case $\mathbf{c_n>0}$:}
 In this case the interlacing result (Equation~\ref{eq:interlaceone}) tells us that $\lambda_1\geq \mu_1\geq\lambda_i$, $\forall i>1$. Thus we have,
 \begin{align}
&\frac{(\e^T\vv_1)^2}{\lambda_1-\mu_1}-\sum_{i>1}\frac{(\e^T\vv_i)^2}{\mu_1-\lambda_i}=\frac{1}{c_n}\nonumber\\
&\lambda_1-\mu_1\leq c_n (\e^T\vv_1)^2\label{eq:bound-cn-neg}
\end{align}

\paragraph{Case $\mathbf{c_n<0}$:}
 In this case the interlacing result (Equation~\ref{eq:interlacetwo}) tells us that $0\leq \mu_1- \lambda_1\leq\mu_1-\lambda_i$, $\forall i>1$. 
We now divide the eigenvalues $\lambda_i$ into two groups, one with $\mu_1-\lambda_i\leq 2|c_n|$ (denoted by $S_{|c_n|}$), and $\mu_1-\lambda_i> 2|c_n|$. Since $\sum_i(\vv_i^T\e)^2=1$, we have:
\begin{align}
\frac{1}{|c_n|}&=\sum_i\frac{(\e^T\vv_i)^2}{\mu_1-\lambda_i} \leq \lsum_{i\in S_{|c_n|}} \frac{(\e^T\vv_i)^2}{\mu_1-\lambda_i}+\frac{1}{2|c_n|}\nonumber
\end{align}
Further, since $\mu_1-\lambda_1\leq \mu_1-\lambda_i$, $\forall i>1$,
\begin{align}
\mu_1-\lambda_1\leq 2|c_n| \lsum_{i\in S_{|c_n|}} (\e^T\vv_i)^2
\label{eq:bound-cn-pos}
\end{align}

Let $c_n^-$ equal $-c_n1(c_n<0)$.
Combining Equations~\ref{eq:bound-cn-neg} and~\ref{eq:bound-cn-pos} we see that:
\begin{align}\label{eq:case-by-case}
|\lambda_1-\mu_1|\leq |c_n|\max(2\lsum_{i\in S_{c_n^-}} (\e^T\vv_i)^2, (\e^T\vv_1)^2).
\end{align}



We now invoke Result~\ref{res:local-semicircle} to bound the size of $S_{c_n^-}$. We use Equation~\ref{eq:weyl-dev}, and Result~\ref{res:maxeig} to note that, $|\mu_1|> 5$ with probability tending to zero as $n\rightarrow\infty$, and hence we can apply Result~\ref{res:local-semicircle}.
\begin{figure}
\caption{\label{fig:semicircle}The semicircle distribution}
\hspace{10em}\includegraphics[scale=.3]{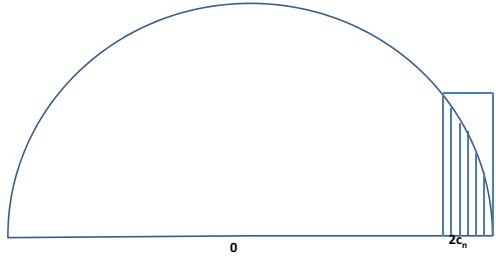}
\end{figure}
Let $b=\mu_1$, and $a=\mu_1-2c_n^-$.

Clearly, we have $\int_{a}^{b}\rho_{sc}(x)dx\leq \int_{a}^{2}\rho_{sc}(x)dx$, since $\rho_{sc}(x)=0$ for $x>2$.
Now, from Equation~\ref{eq:semicircle} we see that $\int_{a}^{2}\rho_{sc}(x)dx$ is proportional to the area of the shaded region in Figure~\ref{fig:semicircle}, which can be upper bounded by the area of a rectangle having sides of length $2c_n^-$ and $\sqrt{4-(2-2c_n^-)^2}$.
 Hence,
 $\int_{a}^{b}\rho_{sc}(x)dx\leq \int_{a}^{2}\rho_{sc}(x)dx\leq C_2 (c_n^-)^{3/2}$ for some positive constant $C_2$. Now, for $p$ fixed w.r.t $n$, $c_n^-=O_P(n^{-1/2})$, which together with Equation~\ref{eq:local-density-er1} yields:
 \begin{align}
 \label{eq:nab}
|S_{c_n^-}|= \mathcal{N}(a,b)=\mathcal{N}_{sc}(a,b)+O_P(\log n)^{L}= O_P(n^{1/4}).
 \end{align}

 Finally, we can invoke Result~\ref{res:eigvec-proj} to obtain:
 \begin{align}
 \label{eq:final-step-cn-pos}
\lsum_{i\in S_{c_n^-}} (\e^T\vv_i)^2 =\tilde{O}_P(n^{-3/4})
 \end{align}

Since $(\e^T\vv_1)^2=\n{O}_P(1/n)$ using Result~\ref{res:eigvec-proj}, Equation~\ref{eq:case-by-case} in conjunction with Equation~\ref{eq:final-step-cn-pos} yields $|\lambda_1-\mu_1|=\n{O}_P(n^{-5/4})$.
The  $\tilde{O}$ notation in Definition~\ref{def:nO} ensures that $\tilde{O}_P(n^{-5/4})$ is $o_P(1/n)$ for large enough $n$.
%

\end{proof}
Finally we are ready to prove our main result.
\subsection{Proof of Theorem~\ref{thm:TW}}
\begin{proof}
We proceed in two steps. First we consider the matrix  $\n{A}_1'=(A-\hat{P})/\sqrt{(n-1)p(1-p)}$. We note that:
$$\n{A}_1'-\n{A}_1=\frac{(\ph-p)I}{\sqrt{(n-1)p(1-p)}}.$$
Since $p$ is a constant w.r.t $n$, using Lemma~\ref{lem:tw1} we have:
\begin{align}
\label{eq:naoneprime}
\lambda_1(\n{A}_1')=\mu_1+O_P(n^{-3/2})=\lambda_1+o_P(1/n).
\end{align}

Finally, we note that $\lambda_1(\natwo)=\lambda_1(\n{A}_1')\sqrt{\frac{p(1-p)}{\ph(1-\ph)}}$. A simple application of the Chernoff bound gives $\sqrt{\ph(1-\ph)}=\sqrt{p(1-p)}(1+O_P(1/n))$, and  thus using Equation~\ref{eq:naoneprime} we have,
 $$\lambda_1(\natwo)=\lambda_1+o_P(n^{-2/3}).$$
 However, Equation~\ref{eq:TW} establishes the edge universality of $\n{A}$, thus yielding the final result.
\end{proof}

We conclude this section with a proof of Lemma~\ref{lem:sbm}.
\subsection{Proof of Lemma~\ref{lem:sbm}.}
\begin{proof}
If $B_{ii}> \sum_{j\neq i} B_{ij}$, then $B$ is a positive definite matrix by diagonal dominance. Hence, $ZBZ^T$ is also positive definite. Since we are considering the dense regime of degrees, i.e. where the elements of $B$ are constant w.r.t $n$, the  $k$ largest eigenvalues of $E[A|Z]$ (Equation~\ref{eq:sbmdef}) are of the form $C_i n$, where $C_i, 1\leq i\leq k$, are positive constants. ~\citet{Oliveira2010} show that $\lambda_i(A)=\lambda_i(E[A|Z])+O_P(\sqrt{n\log n})$. Hence with high probability, the $k$ largest eigenvalues of $A$ will be positive. Using Weyl's identity we have $\lambda_{2}(A)\leq \lambda_1(A-\hat{P})\leq \lambda_{1}(A)$. Thus with high probability $\lambda_1(A-\hat{P})\geq Cn$ for some positive constant $C$. Thus for large $n$, $\lambda_1(\natwo)\geq C'\sqrt{n}$ w.h.p, and thus the result is proved.
\end{proof}
\section{Experiments}
In this section we present experiments on simulated and real data to demonstrate the performance of our method. We use simulated data to demonstrate two properties of our hypothesis test. First we show that it can differentiate an \er graph from another with a small cluster planted in it, namely a \sbm with one class much smaller in size that the other. Secondly we show that, while our theory only holds for probability of linkage $p$ fixed w.r.t $n$ (i.e. the case with degree growing linearly with $n$), our algorithm works for sparse graphs as well.
\subsection{Hypothesis Test}
 Using the same setup as~\citet{levina_pnas} we plant a densely connected small cluster in an \er graph. In essence we are looking at a stochastic blockmodel with $n=1000$, and $n_1$ nodes in cluster one. The block model parameters are $B_{11}=0.15$, $B_{22}=B_{12}$. We  plot error-bars from fifty random runs on the p-values against increasing $n_1$ values in Figure~\ref{fig:hptest}(A) and p-values against increasing $B_{12}$ values in Figure~\ref{fig:hptest}(B).
A larger p-value simply means that the hypothesis test considers the graph to be close to an \er graph. In Figure~\ref{fig:hptest}(A) we see that the p-values decrease as $n_1$ increases from thirty to a hundred. This is expected since the planted cluster is easier to detect as $n_1$ grows. On the other hand, in Figure~\ref{fig:hptest}(B) we see that the p-values increase as $P_{12}$ is increased from 0.04 to 0.1. This is also expected since the graph is indeed losing its block structure.
\begin{figure}[!htb]
\caption{\label{fig:hptest}We plot p-values computed using Algorithm~\ref{alg:hptest} in simulated networks of $n=1000$. On the left panel, $B_{11}=0.15$ and $B_{12}=B_{22}=0.05$. }
 \begin{tabular}{@{\hspace{0em}}c@{\hspace{0em}}c}
\includegraphics[width=0.5\linewidth]{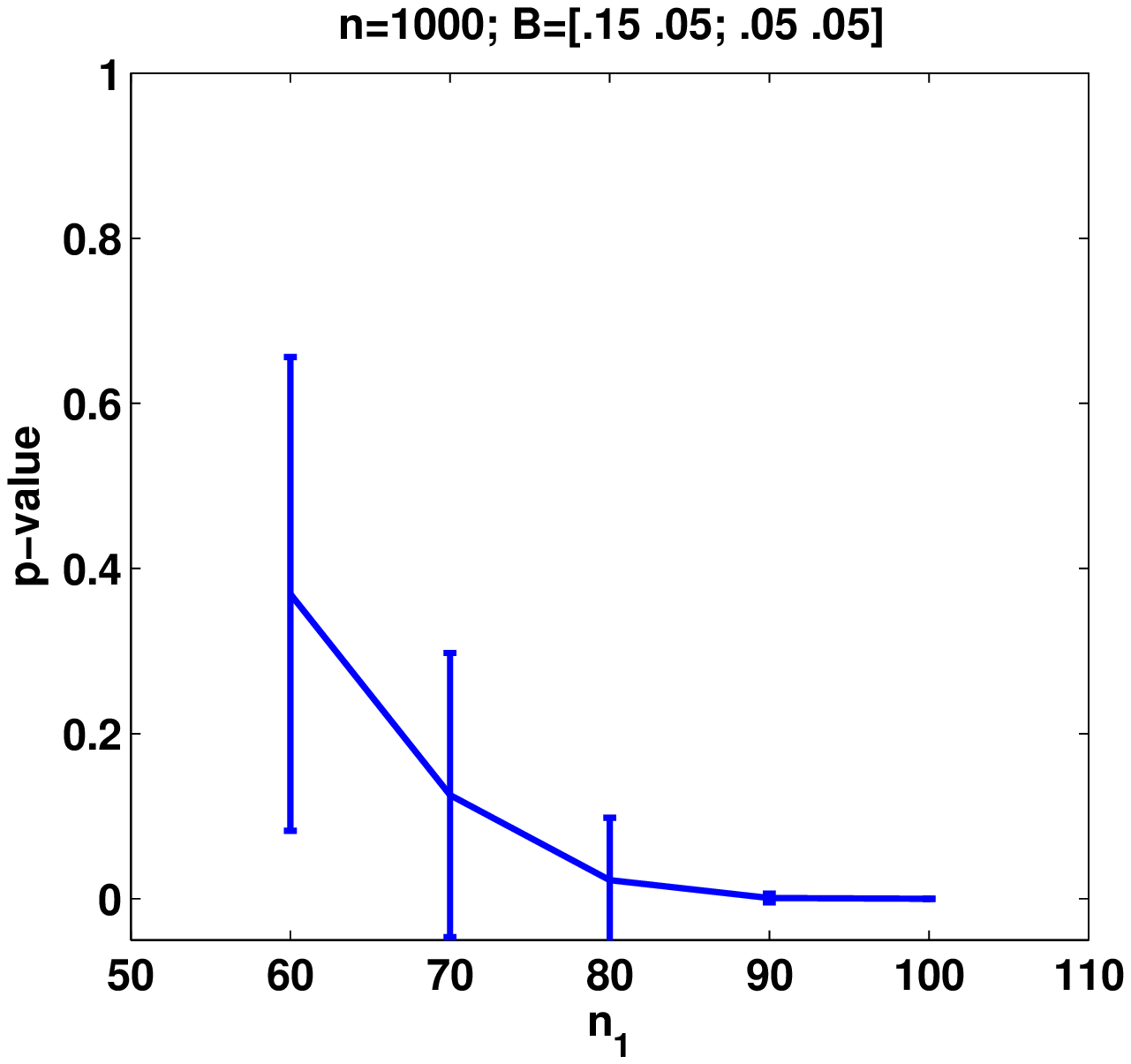}&
\includegraphics[width=0.55\linewidth]{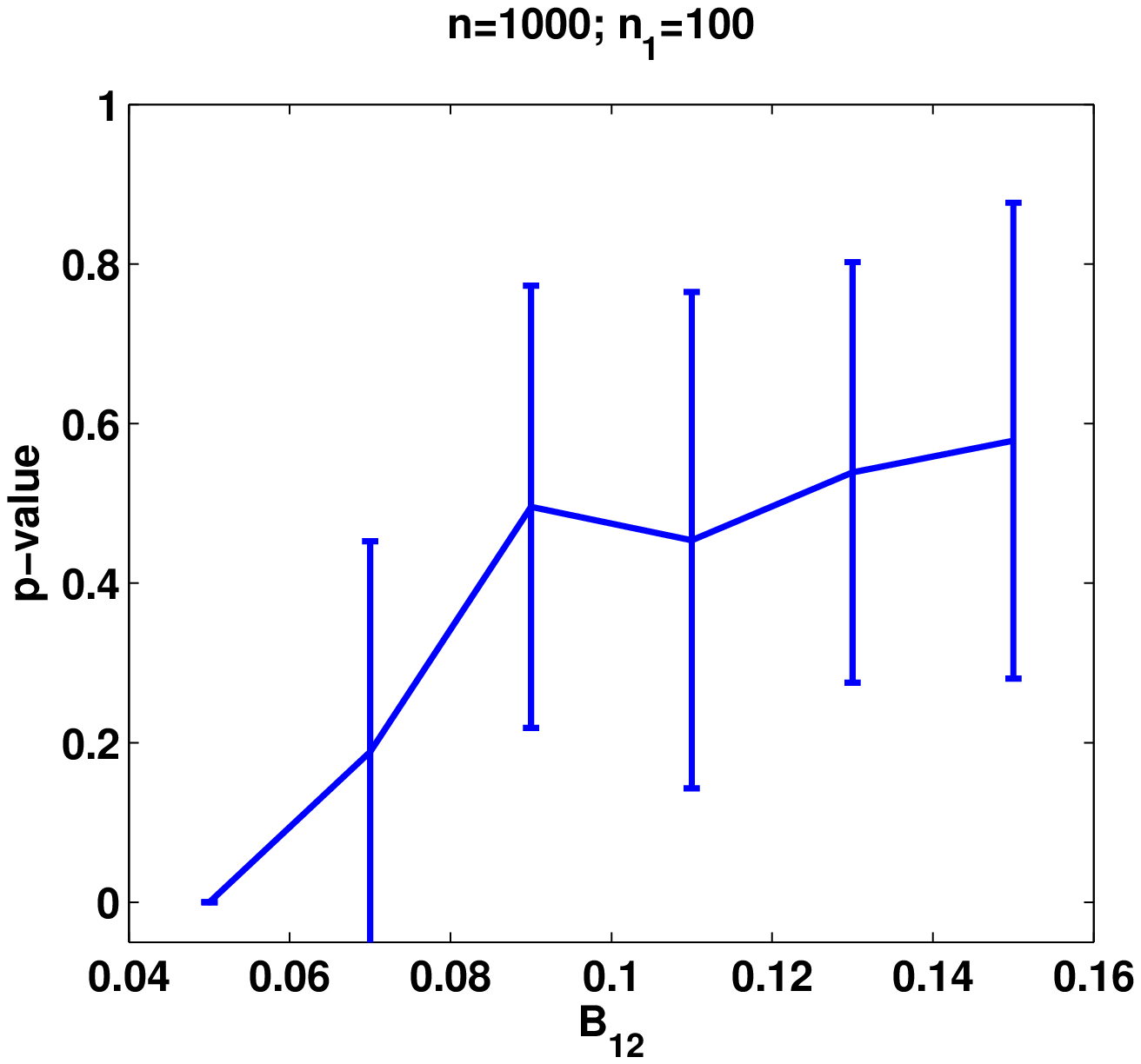}\\
(A)&(B)
\end{tabular}
\end{figure}
\subsection{Nested Stochastic Blockmodels}
We present a ``nested'' \sbm, where the communities become increasingly dense. Specifically, $B_{11}=B_{22}=\rho a$, $B_{12}=\rho b$, $B_{13}=B_{23}=\rho c$, and $B_{33}=\rho d$, where $a=0.2$, $b=0.1$, and $c=0.01$. As we increase $\rho$ from 0.05 to 0.25 in steps of 0.05, the average expected degree of a $n=1000$ node graph increases from $2.8$ to $13.8$. We plot errorbars on p-values from fifty random runs. Similar to~\citep{levina_pnas} we use the Adjusted Rand Index, which is a well known measure of closeness between two sets of clusterings with $n_1=n_2=200$ and $n_3=600$.

\begin{figure}[!ht]
  \begin{minipage}[t]{.5\linewidth}
  \par\vspace{0pt}
    \centering
\includegraphics[width=\linewidth]{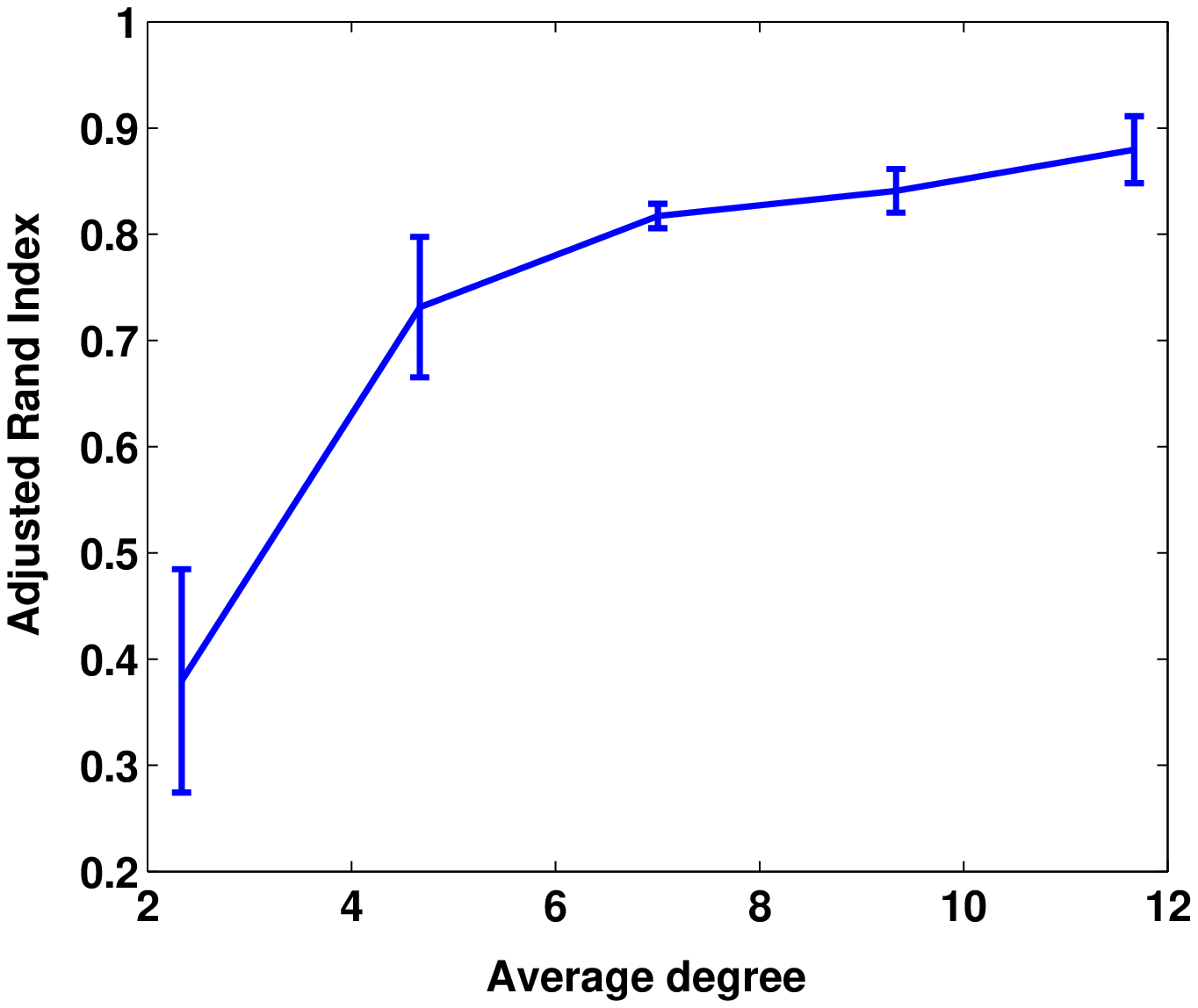}
\vspace{-2em}
\caption{\label{fig:nested-bm}We plot Adjusted Rand Index averaged over fifty random runs. Higher the value, closer the clustering to the true labels.}
  \end{minipage}
  \hspace{1em}
  \begin{minipage}[t]{.4\linewidth}
  \par\vspace{60pt}
\begin{tabular}{| l | c |} \hline
Algorithm&Adjusted Rand Index \\\hline
\texttt{E} & 0.55$\pm$ 0.03\\
\rb & 0.88$\pm$ 0.03 \\
\hline
\end{tabular}
\vspace{35pt}
\caption{Comparison with the Community Extraction algorithm (\CE) averaged over fifty random runs\label{fig:compare_nested}}
\end{minipage}
  \end{figure}
%
Figure~\ref{fig:nested-bm} shows that the Adjusted Rand Index grows as the average degree increases. This also demonstrates that while theory holds only for fixed $p$ w.r.t $n$, in practice our recursive bipartitioning algorithm works for sparse graphs as well. We used a p-value cutoff of 0.01 for the simulation experiments.

Finally, we compare our method with~\citet{levina_pnas}.  In Figure~\ref{fig:compare_nested} we show the ARI score obtained using \texttt{E} and \rb for our nested block model setting with the largest expected degree. In this particular case, \CE first extracts the community containing communities one and two, and then tries to extract another community from the remainder of the graph, leading to poor performance. This accuracy can be improved by changing their ``sequential'' extraction strategy with a recursive one.
\subsection{Facebook Ego Networks}
We show our results on ego networks manually collected and labeled by~\citet{McALes12}. Here we have a collection of nine networks which are induced subgraphs formed by neighbors of a node. The central node is called the ego node. The ground truth labels consist of overlapping cluster assignments, also known as circles. The hope is to identify social circles of the ego node by examining the network structure and features on nodes. While~\citet{McALes12}'s work takes node features into account, we only work with the network structure. For every network we remove nodes with zero degree, and cluster the remaining nodes. Since ground truth clusters are sometimes incomplete, in the sense that not all nodes are assigned to some cluster, we use the F-score for comparing two clusterings. Consider the ground truth cluster $C$ and the computed cluster $\hat{C}$. The F-measure between these is defined as follows:
\begin{align*}
Recall(C,\hat{C})&=\frac{|C\bigcap\hat{C}|}{|C|}, \qquad Precision(C,\hat{C})=\frac{|C\bigcap\hat{C}|}{|\hat{C}|}\\
F(C,\hat{C})&=\frac{2\times Precision(C,\hat{C})\times Recall(C,\hat{C})}{Precision(C,\hat{C})+Recall(C,\hat{C})}
\end{align*}

This was extended to hierarchical clusterings by~\citet{fmeas_hierarchy}. For ground truth cluster $C_i$, one computes $x_i=\max_j(F(C_i,\hat{C}_j))$, where $\hat{C}_j$ is obtained by flattening out the subtree for node $j$ in the hierarchical clustering tree. Now the overall $F$ measure is obtained by computing an weighted average $\sum_i x_i |C_i|/(\sum_j |C_j|)$.
For the real data we use a cutoff ($\alpha$ in Algorithm~\ref{alg:recursive}) of 0.0001. We can also stop dividing the graph, when the subgraph size falls under a given number, say $n_\beta$. While we report results without any such stopping conditions added, we would like to note that for $n_\beta=10$,  the F-measures are similar, while the number of clusters are fewer.
In Table~\ref{table:facebook} we compare our recursive bipartitioning algorithm (\rb) with~\citet{McALes12} using the code kindly shared by Julian McAuley.
\begin{table}
\hspace{-3em}
\caption{\label{table:facebook}F-measure comparison on nine Facebook ego-networks}
\resizebox{.99\columnwidth}{!}{
\begin{tabular}{|cccccccccc|}
\hline
&&&&&&&&&\\
Nodes with nonzero degree  &  333&1034&224&150&61&786& 747  & 534&52\\[1ex]
Number of Ground truth clusters & 24   &  9  &  14 &    7 &   13 &   17&    46&    32&    17\\[2ex]\hline
Fmeasure (~\citet{McALes12})& 0.33  &  0.25&    0.58&    0.56&    0.49&    0.48&    0.38&   0.15&    0.40\\[2ex]\hline
Number of clusters learned by \rb & 23  &  66 &   20 &   11 &    8  &  60  &  39 &   38  &   6\\[1ex]
Fmeasure& 0.47 &   0.60 &   0.76&    0.79&    0.71&    0.74&    0.63&  0.32&    0.49\\[1ex]\hline
\end{tabular}
}
\end{table}

We see that we obtain better or comparable F-measures for most of the ego networks. In order to visualize the cluster structure uncovered by \rb, we present Figure~\ref{fig:facebook-density-plot}. In this figure we show a density image of a matrix, whose rows and columns are ordered such that all nodes in the  same subtree appear consecutively. Thus nodes in every subtree correspond to a diagonal block in Figure~\ref{fig:facebook-density-plot}(A). Also, a subtree belonging to a parent subtree will give rise to a diagonal block contained inside that of the parent subtree. This helps one to see the hierarchical structure. Further,  we shade every diagonal block using the \ph computed from the subgraph induced by nodes in the subtree corresponding to it.

 In Figure~\ref{fig:facebook-density-plot}(A) we plot this matrix for one of the ego networks in log scale. Lighter the shading in a block, higher the corresponding \ph. In order to match this image with the graph itself, we also plot the adjacency matrix with rows and columns ordered identically in Figure~\ref{fig:facebook-density-plot}(B). The density plot shows that the hierarchical splits find regions of varied densities.
\begin{figure}[h]
\begin{center}
\begin{tabular}{cc}
\includegraphics[width=.4\linewidth]{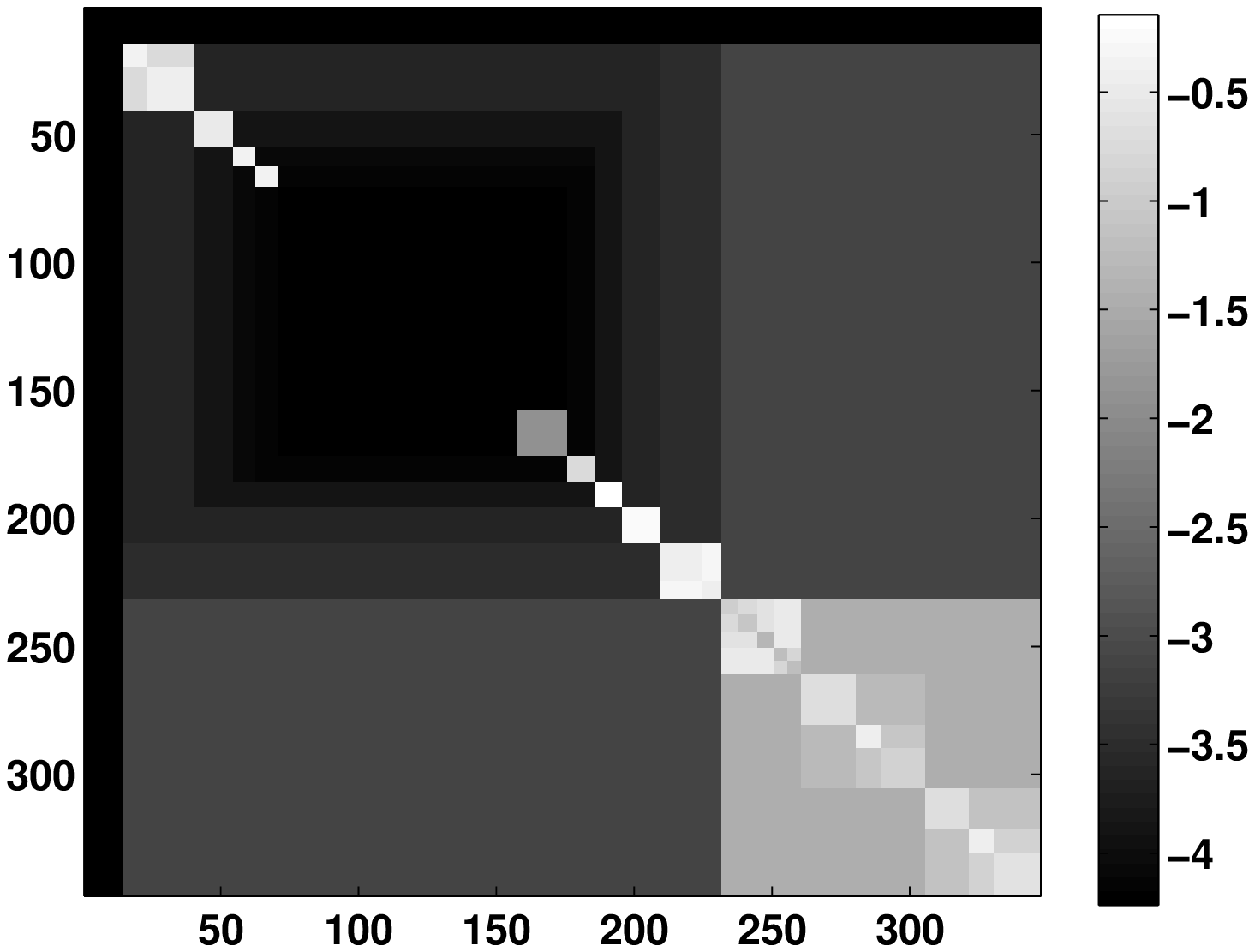}&\includegraphics[width=.4\linewidth]{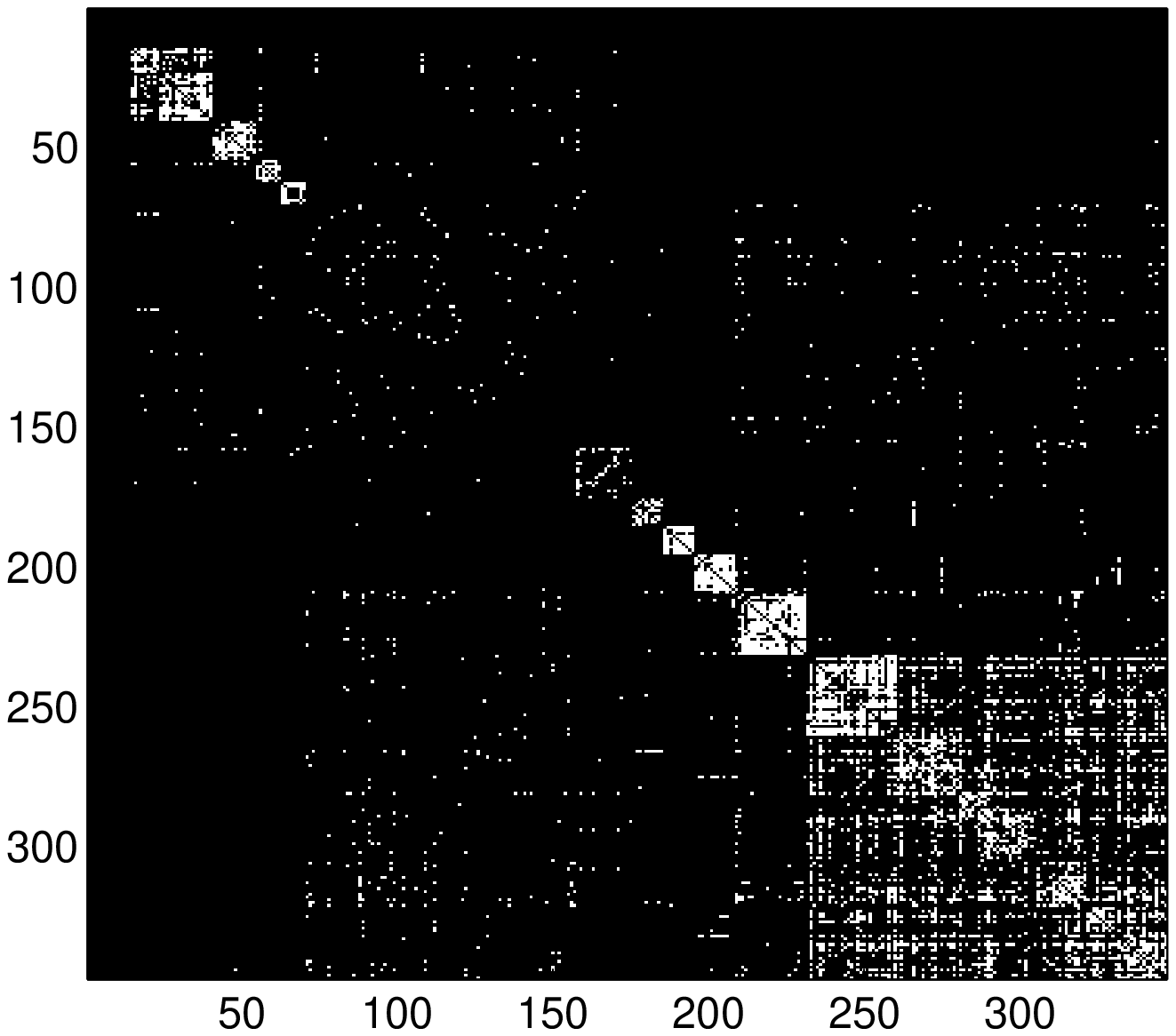}\\
(A)&(B)\\
\end{tabular}
\caption{\label{fig:facebook-density-plot}(A) Density plot with rows ordered to have nodes from the same cluster consecutively. (B) Adjacency Matrix using the same order. }
\end{center}
\end{figure}
\subsection{Karate Club and the Political Books Network}
The Karate Club data is a well known network which has 34 individuals belonging to a karate club. Later the members split into two groups after a disagreement on class fees~\citep{zachary}. These two groups are considered the ground truth communities.
\begin{figure}
\hspace{-3em}
\begin{tabular}{c@{\hspace{-.2em}}c@{\hspace{-.2em}}c@{\hspace{-.2em}}c}
\includegraphics*[width=0.26\linewidth]{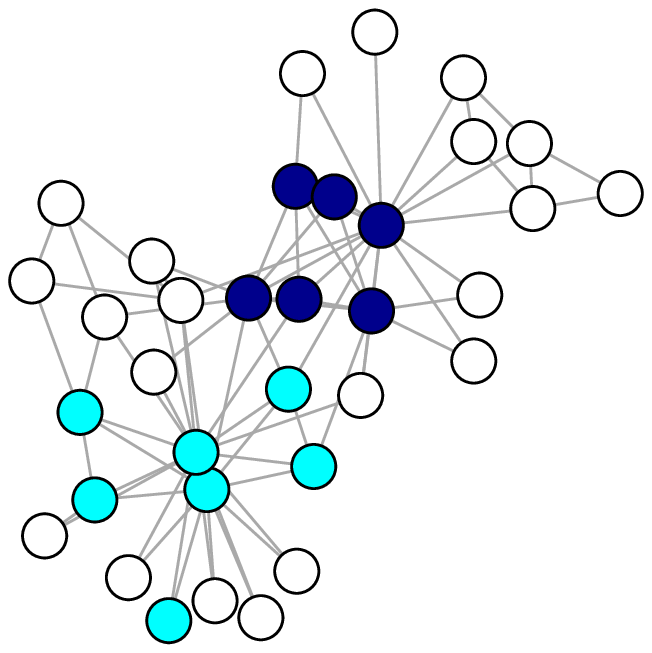}&
\includegraphics*[width=0.26\linewidth]{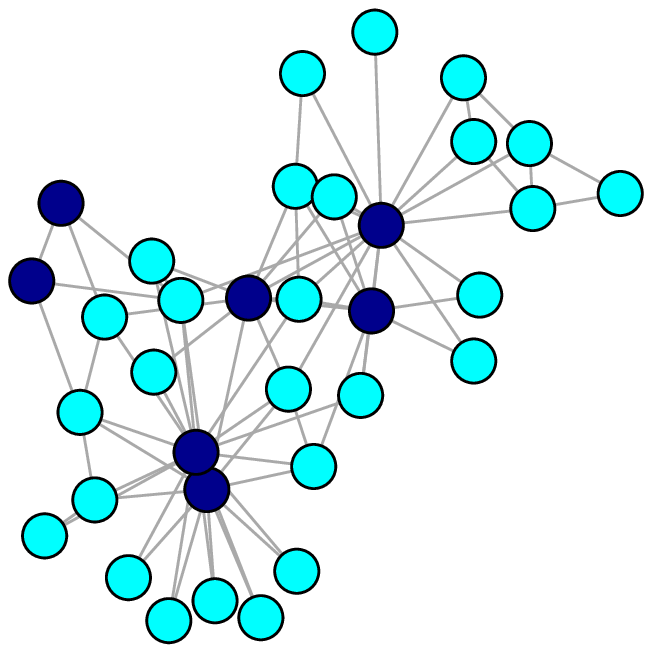}&
\includegraphics*[width=0.26\linewidth]{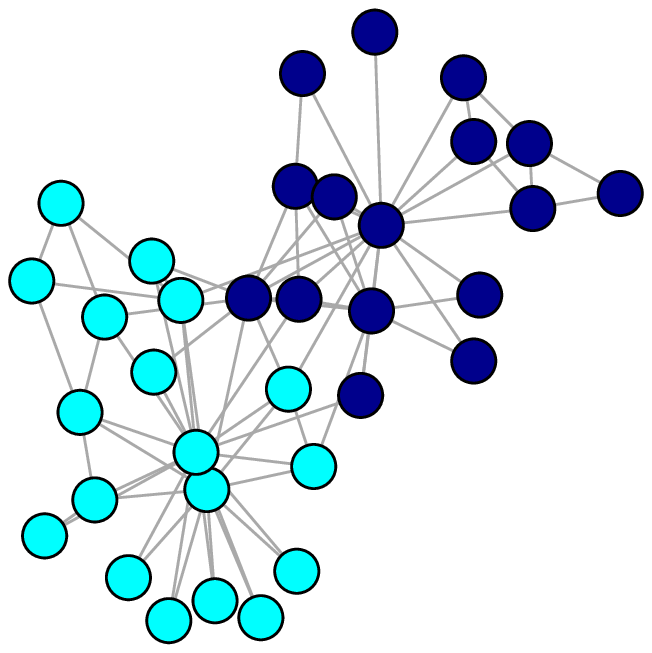}&
\includegraphics*[width=0.26\linewidth]{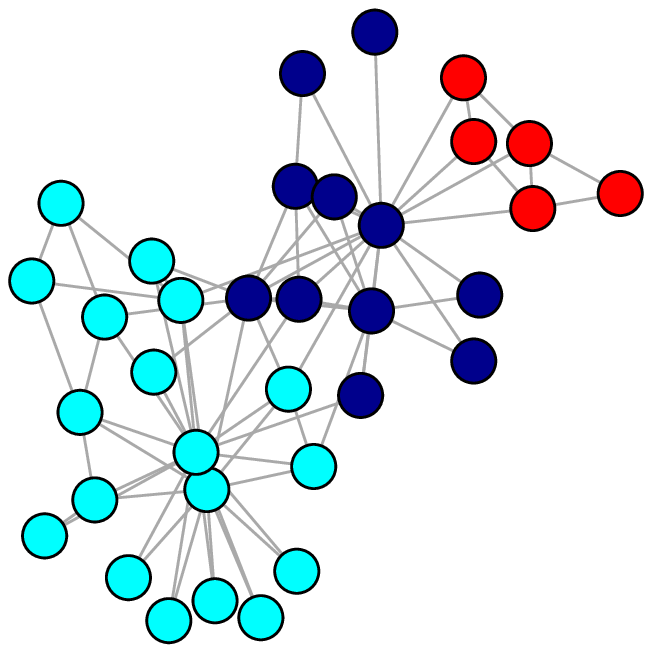}\\
(A)&(B)&(C)&(D)\\
\end{tabular}
\caption{\label{fig:karate}Clusters obtained using (A) Community Extraction, (B) Pseudo Likelihood, (C) Recursive Bipartitioning with p-value cutoff 0.0001 and (D)Recursive Bipartitioning with p-value cutoff 0.01 }
\end{figure}
 We present the clusterings obtained using the different algorithms in Figure~\ref{fig:karate}. In particular, we show the clusterings obtained using the extraction method (\CE) in Figure~\ref{fig:karate}, the Pseudo Likelihood method (\texttt{PL}) with $k=3$~\citep{PL} in Figure~\ref{fig:karate}(B), our recursive bipartitioning algorithm (\rb) using p-value cutoff of $0.0001$ in Figure~\ref{fig:karate}(C), and finally \rb with p-value cutoff of $0.01$ in Figure~\ref{fig:karate}(D). These results are generated using the code kindly shared by Yunpeng Zhao and Aiyou Chen. We see that \CE finds the cores of the two communities, \texttt{PL} puts high degree nodes in one cluster (similar to the MCMC method for fitting a \sbm in~\citet{levina_pnas}). Our method achieves perfect clustering for p-value cutoff of 0.0001. However our statistic computed from the dark blue group has a p-value of about 0.003, which is why we also show the clustering with a larger cutoff. Here the dark blue community is broken further into a clique-like subset of nodes, and the rest. Below we also provide a density plot in Figure~\ref{fig:karate-density-plot} (A) and an image of the adjacency matrix with rows and column ordered similarly to the density plot in Figure~\ref{fig:karate-density-plot} (B) to elucidate this issue.
\begin{figure}[!htb]
\begin{center}
\begin{tabular}{cc}
\includegraphics[width=.4\linewidth]{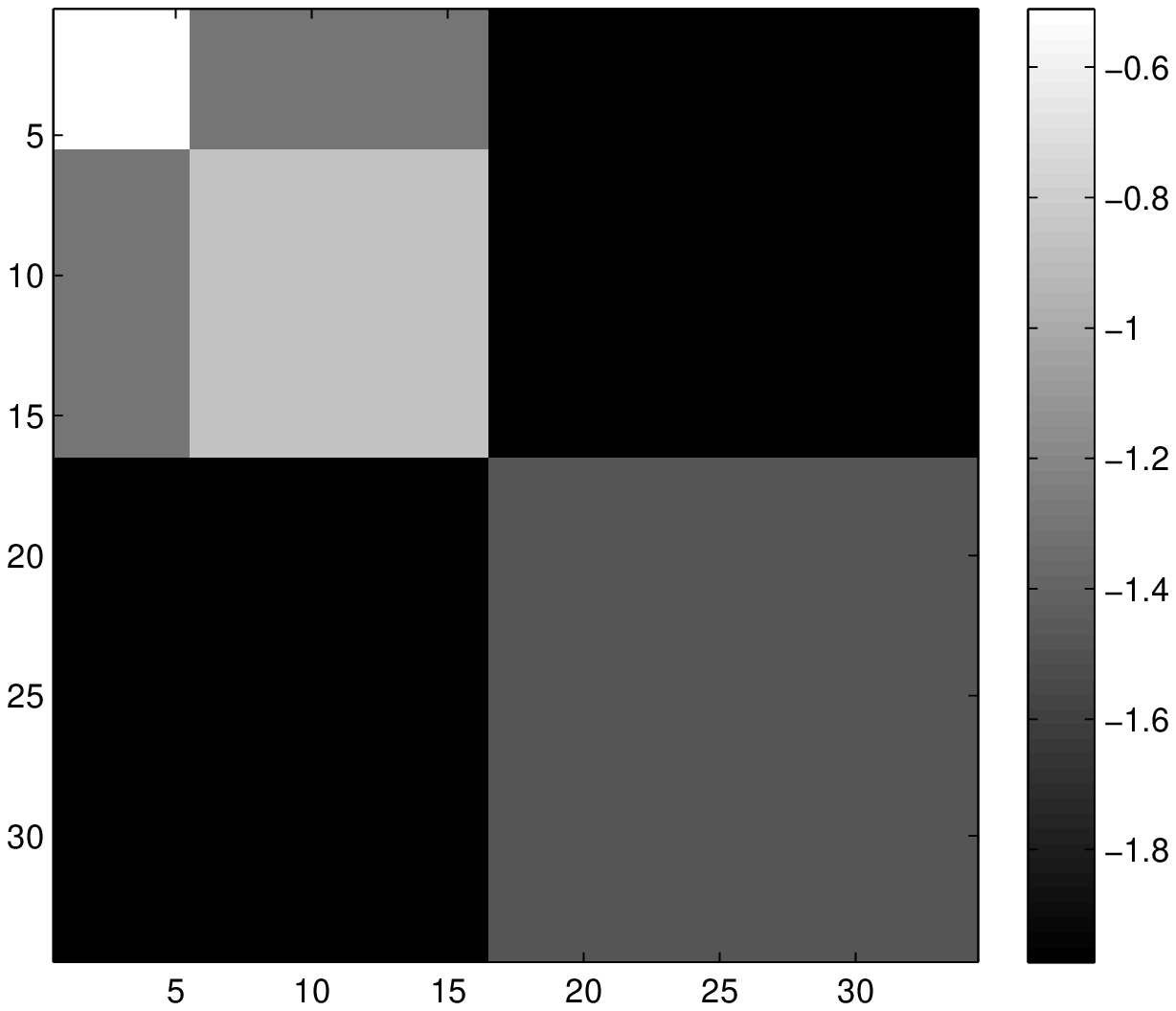}&\includegraphics[width=.37\linewidth]{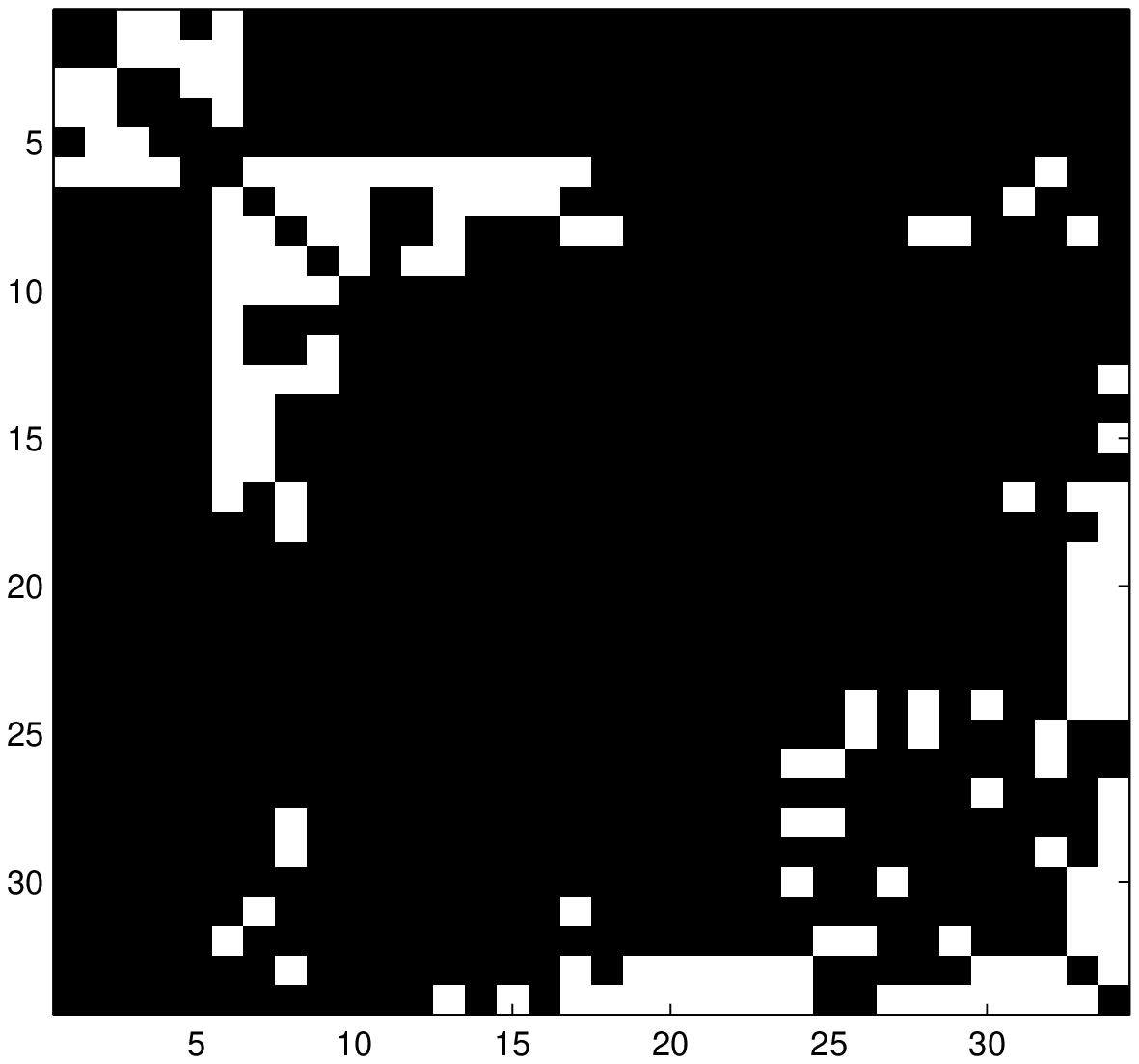}\\
(A)&(B)\\
\end{tabular}
\caption{\label{fig:karate-density-plot}(A) Density plot of the Karate club data with rows ordered to have nodes from the same cluster consecutively. (B) Adjacency Matrix using the same order.}
\end{center}
\end{figure}


The political books network~\citep{newman_polbook} is an undirected network of 105 books. Two books are connected if they are co-purchased frequently on Amazon. While the ground truth is not available on this dataset, the common conjecture~\citep{levina_pnas} is that some books are strongly political, i.e. liberal or conservative, and the others are somewhat in-between. The authors also show that existing algorithms give reasonable results with $k=3$ clusters, and \CE returned the cores of the communities with $k=2$. We show clustering obtained using \texttt{PL} with $k=3$ in Figure~\ref{fig:polbooks}(A), the two communities extracted by the algorithm \texttt{E} in Figure~\ref{fig:polbooks}(B), clustering by \rb in Figure~\ref{fig:polbooks}(C), and finally our density plot in Figure~\ref{fig:polbooks}(D). 

Algorithm \texttt{E} finds the core set of nodes from the green and blue clusters found by \texttt{PL}. \rb on the other hand breaks the graph into six parts. The first split is between the blue nodes with the rest. The second split separates the yellow nodes from the green nodes. The next two splits divide the green nodes and the blue nodes into further smaller clusters.  We overlay the density plot with the row and column reordered adjacency matrix, so that brightest pixels correspond to an edge. The ordering simply puts nodes from the same cluster consecutively, and clusters in the same subtree consecutively. This figure shows the hierarchically nested structure, where we pick up denser subgraphs. 

\begin{figure}[!htb]
 \begin{tabular}{@{\hspace{-2em}}c@{\hspace{-1em}}c@{\hspace{-1em}}c@{\hspace{-0em}}c}
\includegraphics*[width=0.26\linewidth]{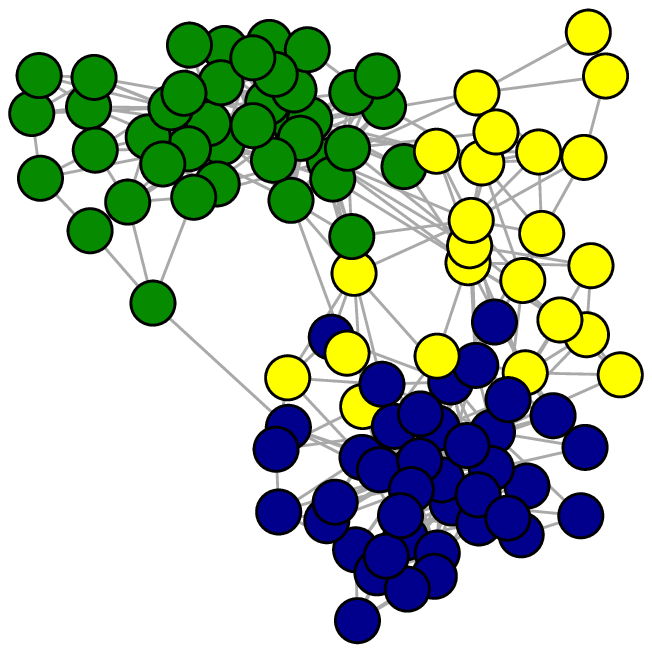}&
\includegraphics*[width=0.25\linewidth]{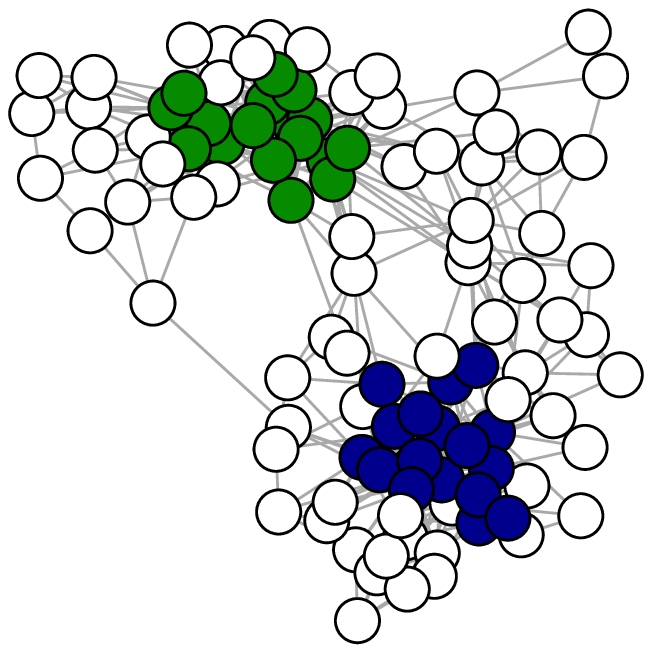}&
\includegraphics*[width=0.25\linewidth]{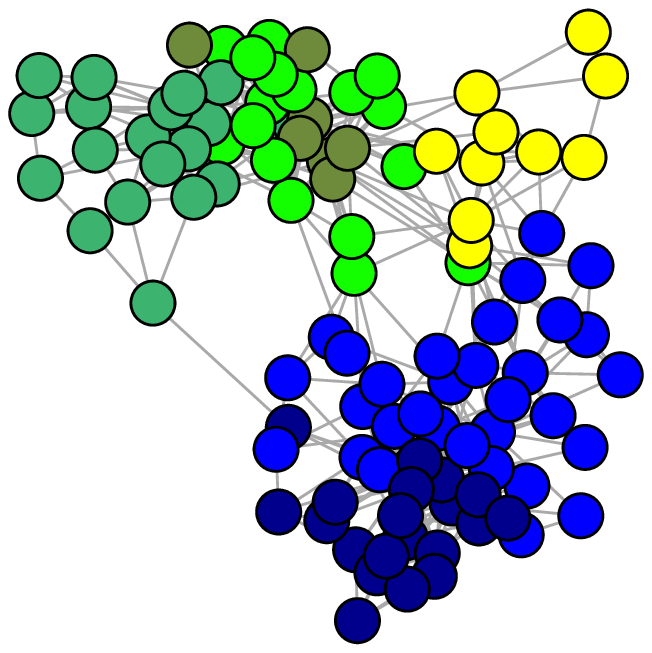}&
  \includegraphics*[width=0.3\linewidth]{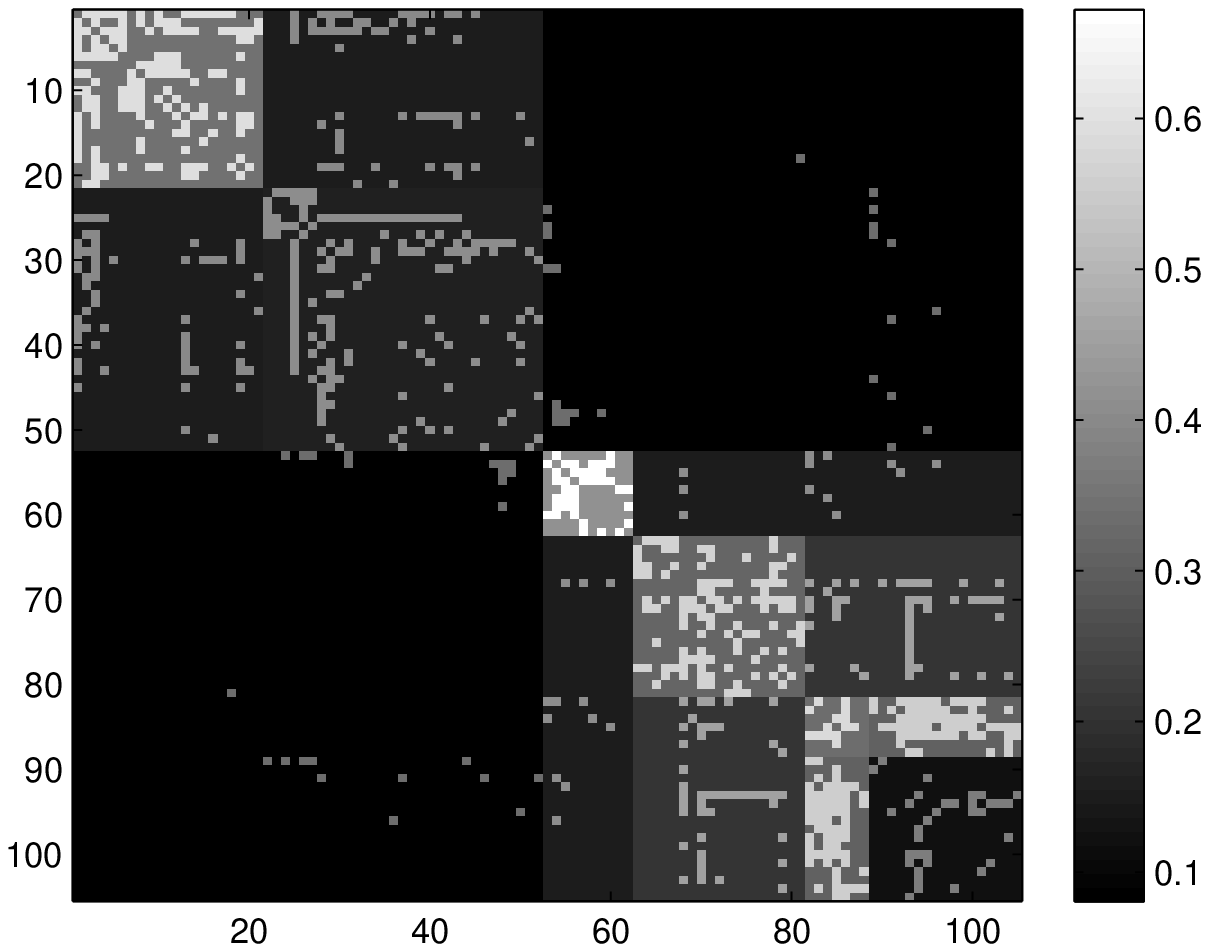}\\
(A)&(B)&(C)&(D)
\end{tabular}
\caption{Clusterings of the Political Books data. (A) \texttt{PL}, (B) \texttt{E}, (C) \rb, and (D) subgraph density plot superimposed with the adjacency matrix.\label{fig:polbooks}}
\end{figure}
\section{Discussion}
In this paper we have proposed an algorithm which provably detects the number of blocks in a graph generated from a \sbm. Using the largest eigenvalue of the suitably shifted and scaled adjacency matrix, we  develop a hypothesis test to decide if the graph is generated from a \sbm with more than one blocks. Our approach is significantly different from existing work because, we theoretically establish the limiting distribution of the statistic under the null, which in our case is that the graph is \er. We also propose to obtain small sample corrections on the limiting distribution, which together with the known form of the limiting law, alleviates the need for expensive parametric bootstrap replicates. Using this hypothesis test we design a recursive bipartitioning algorithm (\rb) which naturally yields a hierarchical cluster structure.

On nine real datasets with ground truth from Facebook, \rb outperforms the existing method that has been shown to have the best performance among other state of the art algorithms for finding overlapping clusters. We also show the nested cluster structure of varied densities discovered by \rb on the karate club data and the political books data.
We would like to point out that our algorithm is not a new clustering algorithm, and one can easily replace the spectral clustering step with some other method, possibly \texttt{E} or \texttt{PL}. Our experiments on the karate club and political books network is not aimed at showing that we find better quality clusters, but that we find interesting structure matching with existing work without having to specify $k$. We choose Spectral Clustering because of its good theoretical properties in the context of Blockmodels~\citep{rohe_chatterji_yu} and its computational scalability.

\section{Acknowledgements}
We thank Elizaveta Levina, Yunpeng Zhao, Aiyou Chen and Julian McAuley for sharing their code.
We are also grateful to Antti Knowles for directing us to the relevant literature for applying the result on isotropic delocalization of eigenvectors  to our setting.
This research was funded in part by NSF FRG Grant DMS-1160319.
\bibliographystyle{chicago}
\bibliography{../../latent-common/latex/purna}
\end{document}